\documentclass[10pt,journal]{IEEEtran}
\usepackage{amsmath,amsfonts,amsthm,amsbsy,amssymb}
\usepackage{algorithm}
\usepackage{algpseudocode}
\usepackage{array}
\usepackage[caption=false,font=normalsize,labelfont=sf,textfont=sf]{subfig}
\usepackage{textcomp}
\usepackage{stfloats}
\usepackage{url}
\usepackage{verbatim}
\usepackage{graphicx}
\usepackage{acronym}
\usepackage{cite}
\usepackage{bm}
\usepackage{xcolor}
\usepackage{booktabs}
\usepackage{tabularx}
\usepackage{hyperref}

\acrodef{EVD}[EVD]{eigenvalue decomposition}
\acrodef{MAP}[MAP]{maximum a posteriori}
\acrodef{POP}[POP]{polynomial optimization problem}
\acrodef{QCQP}[QCQP]{quadratically constrained quadratic program}
\acrodefplural{QCQP}[QCQPs]{quadratically constrained quadratic programs}
\acrodef{SDP}[SDP]{semidefinite program}
\acrodefplural{SDP}[SDPs]{semidefinite programs}
\acrodef{SOS}[SOS]{sums-of-squares}
\acrodef{SVD}[SVD]{singular value decomposition}
\acrodef{TLS}[TLS]{truncated least-squares}
\acrodef{LICQ}[LICQ]{linear independence constraint qualification}
\acrodef{PSD}[PSD]{positive semidefinite}
\acrodef{PGO}[PGO]{pose graph optimization}
\acrodef{RS}[RS]{rotation synchronization}
\acrodef{SLAM}[SLAM]{simultaneous localization and mapping}
\acrodef{KKT}[KKT]{Karush--Kuhn--Tucker}
\acrodef{RANSAC}[RANSAC]{Random Sample Consensus}
\acrodef{PMC}[PMC]{Parallel Maximum Clique}
\acrodef{CPCert}[CP-Cert]{Central-Path Certifier}
\acrodef{PCG}[PCG]{preconditioned conjugate gradient}
\acrodef{GTSAM}[GTSAM]{Georgia Tech Smoothing and Mapping}

\newtheorem{theorem}{Theorem}
\newtheorem{lemma}[theorem]{Lemma}

\newtheorem{proposition}[theorem]{Proposition}

\newtheorem{assumption}{Assumption}
\newtheorem{remark}{Remark}

\newcommand{\inner}[2]{\langle#1,#2\rangle}
\newcommand{\Sym}{\mathbb{S}}
\newcommand{\PSD}{\mathbb{S}_+}
\newcommand{\PD}{\mathbb{S}_{++}}
\DeclareMathOperator{\traceop}{tr}
\newcommand{\trace}[1]{\traceop\left(#1\right)}

\DeclareMathOperator{\rankop}{rank}
\newcommand{\rank}[1]{\rankop(#1)}
\DeclareMathOperator{\support}{supp}
\newcommand{\supp}[1]{\support(#1)}
\DeclareMathOperator{\SOop}{SO}
\newcommand{\SO}[1]{\SOop(#1)}
\DeclareMathOperator{\SEop}{SE}
\newcommand{\SE}[1]{\SEop(#1)}
\newcommand{\vect}[1]{\mbox{vec}\left(#1\right)}
\newcommand{\subscr}[1]{\mbox{\scriptsize #1} }
\newcommand{\card}[1]{\vert#1\vert }
\DeclareMathOperator{\diagop}{diag}
\newcommand{\diag}[1]{\diagop\left(#1\right)}
\DeclareMathOperator{\Exp}{Exp}
\DeclareMathOperator{\st}{s.t.}

\def\candx{\hat{\bm{x}}}
\def\candX{\hat{\bm{X}}}

\def\optValSDP{\rho_{\subscr{SDP}}}
\def\optValQCQP{\rho_{\subscr{QCQP}}}
\def\optValAC{\rho_{\subscr{AC}}}
\def\objMSRC{\rho_{\mbox{\scriptsize MSRC}}}

\def\clique{\mathcal{C}}

\newcommand{\LinCons}[1]{\mathcal{A}(#1)}
\newcommand{\LinConsAdj}[1]{\mathcal{A}^*(#1)}
\def\Schur{\bm{D}}
\def\Precond{\tilde{\bm{D}}} 
\def\VecCons{\bar{\bm{B}}}

\def\Affinity{\bm{M}}

\def\SrcSet{\mathcal{P}}
\def\TrgSet{\mathcal{Q}}
\def\MatchSet{\mathcal{M}}

\def\ie{\emph{i.e.}}
\def\eg{\emph{e.g.}}

\newif\ifanonymous
\anonymousfalse

\begin{document}
\title{Following a Unique Path: A Fast Certifier Applied to Outlier-Robust Pose Registration}

\ifanonymous
\author{Author(s)' names removed for anonymization.}
\else
\author{Connor Holmes~\IEEEmembership{Student Member, IEEE} \and Abhishek Goudar~\IEEEmembership{Member, IEEE},\\ Timothy D. Barfoot~\IEEEmembership{Fellow, IEEE}
\thanks{This work was supported in part by the Natural Sciences and Engineering Research Council of Canada (NSERC).}
\thanks{CH, AG, and TDB are with the University of Toronto Robotics Institute, University of Toronto, Toronto, Ontario, Canada. Corresponding author: {\tt\footnotesize connor.holmes@mail.utoronto.ca}}
}
\fi



\markboth{Draft, August~2026}%
{How to Use the IEEEtran \LaTeX \ Templates}

\maketitle

\begin{abstract}
Certifiable methods have arisen as a means to guarantee global optimality of solutions to non-convex problems using convex \ac{SDP} relaxations.
The most performant of these methods use a local solver to obtain the candidate solution, and then certify its optimality using efficient linear algebra techniques.
However, for many problems of interest in robotics, this local-solve-then-certify approach is impeded by a form of degeneracy in the relaxation, leaving a costly optimization of the relaxation as the only recourse.
In this paper, we introduce our \ac{CPCert}, a certifiable method explicitly tailored to certify candidate optima to problems that exhibit this form of degeneracy. Using a candidate as a starting point, our approach seeks a nearby region of the feasible space -- known as the \emph{central path} -- where a valid certificate can be readily obtained. The approach is kept efficient by exploiting indirect linear algebra techniques, problem sparsity, and parallelism.
We apply \ac{CPCert} to both matrix-weighted pose registration and pointcloud data association, whose novel \ac{SDP} relaxation is of independent interest. On simulated examples, we explore the properties of this novel relaxation and show that \ac{CPCert} is fast and scalable, achieving runtimes that are up to three orders of magnitude faster than state-of-the-art direct solvers. Finally, we combine these contributions into a certifiable, outlier-robust pose-estimation pipeline, which we apply to real-world data.
\end{abstract}

\begin{IEEEkeywords}
Certifiable Methods, Data Association, Pose Registration
\end{IEEEkeywords}

\section{Introduction}\label{sec:intro}

\IEEEPARstart{T}{hough} non-convex optimization is a pervasive element of all areas of modern autonomy, it introduces a dangerous vulnerability into the robotic software stack. The necessity of low latency has driven the community to reliance on \emph{local} optimization techniques, which, although performant, can converge to spurious local optima. These local optima are rarely benign and can lead to catastrophic failure if left unchecked. Traditionally, this issue has been mitigated by establishing heuristic initializations that encourage convergence to the true, global optimum.

In recent years, a body of work has arisen to address this local-optimum vulnerability more directly; \emph{certifiable methods} use convex relaxations to endow a given optimization solution with a certificate of global optimality.
These methods have now been applied to a broad range of problems throughout robotics and computer vision, including pose graph optimization~\cite{rosenSESyncCertifiablyCorrect2019a,tianAsynchronousParallelDistributed2020,carloneLagrangianDuality3D2015a,brialesCartanSyncFastGlobal2017a}, rotation averaging~\cite{dellaertShonanRotationAveraging2020,erikssonRotationAveragingStrong2018}, essential matrix estimation~\cite{garcia-salgueroFastRobustCertifiable2021,zhaoEfficientSolutionNonMinimal2020}, sensor and robot calibration~\cite{wiseCertifiablyCorrectAlgorithm2026,songCertifiableAlignmentGNSS2026}, pointcloud regression and localization~\cite{yangTEASERFastCertifiable2021a,korotkineGloballyOptimalDataAssociationFree2025b}, trajectory optimization~\cite{kangGlobalContactRichPlanning2025a}, model predictive control~\cite{mahajanTinySDPRealTime}, and many more~\cite{yangCertifiablyOptimalOutlierRobust2023a,zhaoAdvancesGlobalSolvers2026}.

At present, one of the most commonly used relaxations in the community is known as \emph{Shor's relaxation}, which relaxes \acp{QCQP} to convex \acp{SDP}. The popularity of this relaxation is, in part, due to its versatility; any problem with polynomial cost and constraints can be represented as a \ac{QCQP}.

\begin{figure}[t]
    \centering
    \includegraphics[width=\columnwidth]{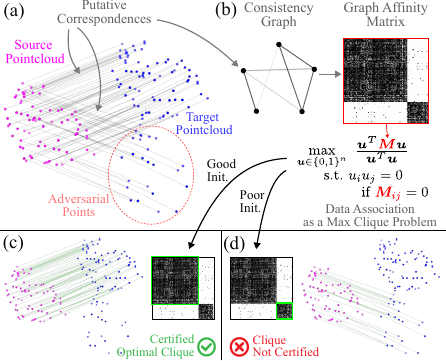}
    \caption{Choosing optimal correspondences can be posed as a certifiable maximum clique optimization problem using consistency graphs. (a) shows putative correspondences (gray) between a source (magenta) and target (blue) pointcloud for an adversarial case. In particular, the target pointcloud contains two copies of the source pointcloud under different transformations, with equal noise but unequal proportions of the total points (150 and 50, respectively). In (b), the correspondences are converted to a consistency graph, whose affinity matrix is used to set up data association as a weighted maximum clique optimization. A local solver for this problem (\eg, CLIPPER) can converge to a local or global optimum, depending on how it is initialized. In (c), we show the clique with correct correspondences, which is certified as globally optimal by our \ac{CPCert} method. In (d), we intentionally initialize the local solver poorly to show that it can converge to a local minimum clique, which is not certified by \ac{CPCert}.}
    \label{fig:adv-example}
\end{figure}

When applying certifiable methods to a problem, a core challenge involves finding a convex relaxation that is \emph{tight} (sometimes called \emph{cost tight}~\cite{dumbgenGloballyOptimalState2024}), meaning that it has the same optimal value as its non-convex counterpart. This paper is particularly concerned with tight \ac{SDP} relaxations whose solutions additionally have a maximum rank of one, which we refer to as \emph{rank-tight} relaxations~\cite{dumbgenGloballyOptimalState2024}.

Recent works have established means of identifying so-called `redundant' constraints that tighten the relaxation~\cite{dumbgenGloballyOptimalState2024} or a hierarchy of relaxations that converge to a rank-tight relaxation~\cite{lasserreMomentSOSHierarchyApplications2024}. These approaches have been deployed to good effect on an expanding set of robotics problems including outlier-robust~\cite{yangCertifiablyOptimalOutlierRobust2023a}, matrix-weighted~\cite{holmesSemidefiniteRelaxationsMatrixWeighted2024a}, and continuous-time~\cite{barfootCertifiablyOptimalRotation2024a} state estimation. Unfortunately, these approaches typically increase the number of constraints and/or dimension of the \ac{SDP} relaxation, which exacerbates already-existing scalability issues.

Performant certifiable methods typically apply fast, local optimization to obtain a candidate primal solution, then certify optimality by finding a uniquely defined set of \emph{dual certificate variables} via efficient linear algebra techniques~\cite{rosenSESyncCertifiablyCorrect2019a,garcia-salgueroFastRobustCertifiable2021}.\footnote{For formal definitions of primal and dual variables, see~\cite{boydConvexOptimization2004a,nocedalNumericalOptimization2006b}.} However, by their very nature, \ac{SDP} relaxations -- especially those that have been tightened using redundant constraints -- often exhibit a form of \emph{degeneracy}, which renders these dual certificate variables non-unique. Currently, the only recourse in these degenerate cases is to solve the \ac{SDP} directly. Due to their robustness, \emph{interior-point solvers} have been established as the \emph{de facto} solvers in this case, but are often too slow to be applied in robotics applications that involve more than a few hundred variables~\cite{yangCertifiablyOptimalOutlierRobust2023a}.

Consequently, though prior methods have established ways to generate tight relaxations for a broad class of robotics problems, the degeneracy of these relaxations has led to an intractability gap in robotics settings~\cite{dumbgenGloballyOptimalState2024,yangCertifiablyOptimalOutlierRobust2023a}.
In this paper, we seek to fill this gap by adopting the local-solve-then-certify approach; we introduce a fast and scalable method for certifying global optimality of candidate solutions to non-convex programs with rank-tight \ac{SDP} relaxations, even when these relaxations are degenerate.
Our core idea is to search for a region of the \ac{SDP} feasible space -- known as the \emph{central path} -- where the dual variables are uniquely defined.
Crucially, we leverage our knowledge of the candidate solution at every step of the algorithm to maximize efficiency and scalability.
Moreover, we adopt an indirect, preconditioned conjugate-gradient approach that avoids repeated construction, manipulation, and factorization of the large matrices, which bottleneck the existing interior-point solvers mentioned above~\cite{tohSolvingLargeScale2004}.
To the best of our knowledge, \ac{CPCert} is the first certifier able to exploit a candidate solution to certify large, degenerate relaxations in near real time without re-solving the \ac{SDP} from scratch.

More specifically, our contributions in this work are as follows:
\begin{itemize}
    \item We introduce and provide open-source code for \ac{CPCert}, a fast method for certifying global optimality of candidate solutions to \acp{QCQP} with tight \ac{SDP} relaxations. We describe the key mechanism that enables the efficiency of our certifier; we apply a scalable conjugate-gradient approach with a preconditioner based on the candidate solution.
    \item We introduce to the robotics community the first tractable means of certifying the problem of data association using a novel \ac{SDP} relaxation that is based on well-known relaxations in the optimization literature. Though this problem is degenerate, we show that it is empirically tight on robotics problems, making it amenable to certification with \ac{CPCert}.
    \item We show that \ac{CPCert} remains practical for substantially larger relaxations than state-of-the-art interior-point solvers, running faster across all problems that we test. In particular, we apply \ac{CPCert} to a broad range of simulated data-association problems as well as matrix-weighted pose registration~\cite{holmesSemidefiniteRelaxationsMatrixWeighted2024a}. Both of these problems have redundant constraints and are degenerate.
    \item We validate these contributions by combining them into a certifiable and robust pose-estimation pipeline for stereo-camera-based localization, and demonstrate it on real data.
\end{itemize}

The most closely related work to~\ac{CPCert} is Loraine~\cite{habibiLoraineInteriorpointSolver2024}, which applies a similar \ac{PCG} solver and served as one of the inspirations for this work. However, Loraine is a general-purpose interior-point solver for low-rank \acp{SDP} and does not leverage prior knowledge of candidate solutions when initializing or building its preconditioner.

The remainder of this paper is organized as follows. Section~\ref{sec:background} reviews the necessary background on \ac{SDP} relaxations, optimality certification, and interior-point methods. Section~\ref{sec:cp-cert} introduces \ac{CPCert}, our central-path certifier, including its main algorithm, stopping conditions, and an efficient preconditioned linear solver. Section~\ref{sec:applications} applies \ac{CPCert} to certifying pointcloud data association and to certifiable stereo-pointcloud registration. This section also introduces a novel relaxation that can be used to certify data-association problems. Section~\ref{sec:experiments} presents experimental results on simulated pointcloud data and on a stereo pipeline applied to real-world data. Section~\ref{sec:conclusion} provides high-level conclusions of the paper as well as future directions of research.

\section{Background}\label{sec:background}

\subsection{Notation}\label{sec:notation}

We denote vectors with bold lowercase letters, \eg, $\bm{x}\in\mathbb{R}^n$, and matrices with bold uppercase letters, \eg, $\bm{X}\in\mathbb{R}^{n\times n}$; standard (non-bold) letters denote scalars. The meaning of a subscript applied to such a symbol depends on context. A single subscript on a bold symbol, \eg, $\bm{A}_i$, indexes into a sequence or indexed family of vectors or matrices, $\{\bm{A}_i\}_{i=1}^m$. In contrast, a single subscript on the corresponding non-bold letter, \eg, $x_i$, refers to the $i$-th scalar entry of the bold vector $\bm{x}$, while a double subscript on a bold matrix, \eg, $\bm{X}_{ij}$, refers to its $(i,j)$-th scalar entry. For a bold vector and a pair of indices $i\leq j$, we use the slicing notation $[\bm{x}]_{i:j} = (x_i,x_{i+1},\dots,x_j)\in\mathbb{R}^{j-i+1}$ to denote the subvector formed by its $i$-th through $j$-th entries, inclusive. Similarly, for a bold matrix $\bm{B}$, we use $[\bm{B}]_i$ to denote its $i$-th column.
We use
\begin{equation*}
    \Sym^n = \left\{\bm{X}\in\mathbb{R}^{n\times n} : \bm{X}=\bm{X}^T\right\},
\end{equation*}
\begin{equation*}
    \PSD^n = \left\{\bm{X}\in\Sym^n : \bm{X}\succeq\bm{0}\right\},\quad \PD^n = \left\{\bm{X}\in\Sym^n : \bm{X}\succ\bm{0}\right\},
\end{equation*}
to denote the sets of $n\times n$ symmetric, symmetric positive-semidefinite, and symmetric positive-definite matrices, respectively.

For two matrices, $\bm{A},\bm{B}\in\mathbb{R}^{n\times n}$, we make use of the trace inner product, $\inner{\bm{A}}{\bm{B}} = \trace{\bm{A}^T\bm{B}}$, where $\trace{\cdot}$ and $\rank{\cdot}$ denote the usual matrix trace and rank, respectively. We use $\vect{\cdot}$ to denote the vectorization operator, which stacks the columns of a matrix into a single column vector and satisfies the identities
\begin{gather*}
    \trace{\bm{A}^T\bm{B}} = \vect{\bm{A}}^T\vect{\bm{B}}\\
    \vect{\bm{A}\bm{B}\bm{C}} = (\bm{C}^T\otimes\bm{A})\vect{\bm{B}},
\end{gather*}
where $\otimes$ denotes the Kronecker product. For a vector $\bm{x}$, $\supp{\bm{x}} = \left\{i~|~x_i\neq0\right\}$ denotes its \emph{support}, the set of indices of its non-zero entries, and $\card{\cdot}$ denotes the cardinality of a set. Unless otherwise indicated, $\|\cdot\|$ applied to a vector denotes the Euclidean ($2$-)norm, while $\|\cdot\|$ applied to a matrix denotes the Frobenius norm, $\|\bm{A}\| = \sqrt{\trace{\bm{A}^T\bm{A}}}$.

We denote the set of non-negative real numbers by $\mathbb{R}_+ = \left\{x\in\mathbb{R} : x\geq0\right\}$ and non-negative vectors by $\mathbb{R}^n_+ = \left\{\bm{x}\in\mathbb{R}^n : x_i\geq0\right\}$.

Finally, we make use of the special orthogonal and special Euclidean groups,
\begin{equation*}
    \SO{3} = \left\{\bm{R}\in\mathbb{R}^{3\times3} : \bm{R}^T\bm{R}=\bm{I},~\det(\bm{R})=1\right\},
\end{equation*}
\begin{equation*}
    \SE{3} = \left\{\bm{T} = \begin{bmatrix}\bm{R} & \bm{t}\\ \bm{0}^T & 1\end{bmatrix} : \bm{R}\in\SO{3},~\bm{t}\in\mathbb{R}^3\right\},
\end{equation*}
representing, respectively, the group of three-dimensional rotation matrices and the group of three-dimensional rigid-body transformations.

\subsection{SDP Relaxations in Robotics}\label{sec:bg-sdp-robotics}

In this section, we review some well-known background material regarding \ac{SDP} relaxations in robotics and computer vision.\footnote{For the sake of brevity, the concepts introduced in this section have been stated without proof. The interested reader is referred to~\cite{cifuentesLocalStabilitySemidefinite2022a,vandenbergheSemidefiniteProgramming1996a} and references therein for more extensive expositions.} A common starting point when dealing with such relaxations is the following, standard-form \textit{quadratically constrained quadratic program} (\ac{QCQP}):
\begin{equation}\label{opt:QCQP}
	\begin{array}{rl}
		\optValQCQP = \min\limits_{\bm{x}\in\mathbb{R}^n} & \bm{x}^T\bm{C}\bm{x} \\
		\st&\bm{x}^T\bm{A}_{i}\bm{x}= b_i, \quad \forall i =1,\dots,m\\
	\end{array}\tag{QCQP}
\end{equation}
where $\optValQCQP$ is the optimal value, $ \bm{x}\in\mathbb{R}^n $ is the optimization variable, $ \bm{C}\in\Sym^n $ represents the quadratic cost, and $ (\bm{A}_i, b_i) $ ($\bm{A}_i\in\Sym^n$) correspond to $ m $ quadratic constraints. Despite its simplicity, it turns out that any problem with polynomial cost and constraints -- including many problems in robotics -- can be expressed in this form. Moreover, this formulation naturally gives rise to the following convex \ac{SDP} relaxation, known as Shor's relaxation~\cite{shorQuadraticOptimizationProblems1987a}:
\begin{equation}\label{opt:SDP}
	\begin{array}{rl}
		\optValSDP = \min\limits_{\bm{X}\in\PSD^n} & \inner{\bm{C}}{\bm{X}} \\
		\st&\inner{\bm{A}_i}{\bm{X}}= b_i, \quad \forall i =1,\dots,m.
	\end{array}\tag{SDP}
\end{equation}
To see why the latter is a relaxation of the former, consider the fact that if we apply the (non-convex) constraint, $\rank{\bm{X}}=1$, to the relaxation, then we recover the non-convex \ac{QCQP} since:
\begin{equation*}
    \bm{X} \in \PSD^n,~\rank{\bm{X}}=1 \iff \bm{X}=\bm{x}\bm{x}^T.
\end{equation*}
The equivalence of \eqref{opt:SDP} and \eqref{opt:QCQP} then follows from the cyclic property of the trace operator.

Since \eqref{opt:SDP} is a relaxation, we have the following inequality:\footnote{Alternatively, we can show this lower bound property by noting that \eqref{opt:QCQP} and \eqref{opt:SDP} share a convex dual problem and then invoke the property of \emph{weak duality}.}
\begin{equation}\label{eqn:weak-duality}
    \optValQCQP \geq \optValSDP.
\end{equation}
We say that the relaxation is \emph{tight} if this relation holds with equality.
Whenever this is the case, there is at least one rank-one optimum of~\eqref{opt:SDP}, corresponding to the optimum of \eqref{opt:QCQP}.\footnote{If $\bm{x}$ is the solution to~\eqref{opt:QCQP}, then $\bm{X}=\bm{x}\bm{x}^T$ is a rank-one feasible point to~\eqref{opt:SDP} with the same cost.}

Even when the global minimizer of \eqref{opt:QCQP} is unique, the global minimizers of \eqref{opt:SDP} can form an entire (convex) set. As mentioned in Section~\ref{sec:intro}, when the maximum rank of all solutions in the \ac{SDP} solution set is one, we say that the relaxation is \emph{rank tight}, a more restrictive condition than tightness. We will show that this condition is important in our context because it typically implies that the \ac{SDP} solution set collapses to a single point that maps exactly to the global solution of~\eqref{opt:QCQP}.

This rank-tight condition can be checked in practice by applying interior-point methods to the \ac{SDP} because they always converge to a unique point that has \emph{maximal complementarity}, meaning that both the primal and dual solution variables have the highest rank possible~\cite{halickaConvergenceCentralPath2002}. Therefore, if the interior-point solution has a rank of one, the \ac{SDP} is rank tight.

\subsection{Certifying Optimality}\label{sec:bg-cert-opt}

When a \ac{QCQP} has a tight \ac{SDP} relaxation, then there exists a solution, $\candx$, to~\eqref{opt:QCQP}, such that
\begin{equation*}
    \optValQCQP=\candx^T\bm{C}\candx = \inner{\bm{C}}{\candx\candx^T} = \optValSDP.
\end{equation*}
That is, $\candX = \candx\candx^T$ is feasible for~\eqref{opt:SDP}. For any other local optimizer, $\bm{x}$, of~\eqref{opt:QCQP}, we have
\begin{equation}
    \bm{x}^T\bm{C}\bm{x}\geq\optValQCQP=\optValSDP
\end{equation}
so the solution, $\candx$, must be globally optimal.

There are potentially two ways that a globally optimal solution can be found: by directly solving the \ac{SDP} or by certifying a given candidate \ac{QCQP} solution. In both cases, we attempt to show that a pair, $(\candX, \candx)$, satisfies the relationship described above.

\subsubsection{Direct Solve}
In the first approach, we directly solve~\eqref{opt:SDP} -- for instance, using an interior-point solver -- to obtain the solution, $\candX$. If the rank of this solution is one, then we can recover a globally optimal solution to the \ac{QCQP} via factorization, $\candX = \bm{x}^*\bm{x}^{*T}$. As shown above, global optimality is trivially guaranteed since the relaxation optimal cost is exactly the same as the non-convex optimal cost. This is the case whenever the relaxation is rank tight.

On the other hand, if the rank of the solution is greater than one, then we can attempt to recover a solution either by applying the \emph{rank-reduction technique} of~\cite{lemonLowRankSemidefiniteProgramming2016}\footnote{This technique is guaranteed to find the low-rank solution whenever the relaxation is tight. Though the algorithm is technically NP-hard, we have observed empirically that it is efficient in practice.} or by applying a \emph{rounding technique} to recover a near-optimal solution~\cite{rosenSESyncCertifiablyCorrect2019a,yangTEASERFastCertifiable2021a}.

\subsubsection{Local-Solve-Then-Certify}
Since direct \ac{SDP} solvers are typically too computationally heavy for the real-time requirements of robotics, efficient certifiable methods apply the second approach; given an existing candidate solution, $\candx$, we attempt to show that its lifted form, $\candX = \candx \candx^T$, constitutes a solution to~\eqref{opt:SDP}. Since~\eqref{opt:SDP} is convex, we need only show that there exist dual variables $(\bm{H}, \bm{\lambda})$ such that the \ac{KKT} conditions of the convex relaxation are satisfied:
\begin{equation}\label{eqn:SDP-KKT}
    \begin{gathered}
        \candX\in\PSD^n,\quad \inner{\bm{A}_i}{\candX}= b_i,~ \forall i =1,\dots,m,\\
        \bm{H}\in\PSD^n,\quad \bm{H} = \bm{C} + \sum\limits_{i=1}^{m} \bm{A}_i \lambda_i, \\
        \inner{\bm{H}}{\candX}=0.
    \end{gathered}
\end{equation}

$\candX$ satisfies the first set of constraints by construction. Since $\bm{H}\in\PSD^n$, the last two constraints form a linear equation in $\bm{\lambda}$,
\begin{equation}\label{eqn:lin-eq-dual-mults}
    \inner{\bm{H}}{\candX} = 0 \iff \bm{H}\candx = \bm{C}\candx + \sum\limits_{i=1}^m \lambda_i \bm{A}_i\candx = \bm{0}.
\end{equation}
Assuming that this linear equation has a unique solution for $\bm{\lambda}$, we can use it to construct $\bm{H}$. Certification of the global optimality of $\candx$ then boils down to testing whether the so-called \emph{certificate matrix} is \ac{PSD}. In practice, the optimality conditions are tested using numerical tolerances; once the multipliers, $\bm{\lambda}$, have been found and the certificate matrix, $\bm{H}$, has been constructed, global optimality is claimed if
\begin{equation}
    \vert\inner{\bm{H}}{\candX}\vert = \vert \candx^T \bm{H} \candx \vert \leq \tau_c~\mbox{ and }~\hat{\bm{H}} + \bm{I}\tau_p \in \PSD^n,
\end{equation}
where $\tau_c>0$ and $\tau_p>0$ are small numerical tolerances. Note that the absolute value on the first term of our test is required since $\bm{H}$ is not necessarily \ac{PSD}.

This second technique has been applied to achieve real-time certified solutions to state-estimation problems in robotics~\cite{rosenSESyncCertifiablyCorrect2019a} and computer vision~\cite{garcia-salgueroFastRobustCertifiable2021}.
Unfortunately, even when the relaxation is tight, its low-rank candidate solution is often \emph{primal degenerate}, meaning that infinitely many dual solutions correspond to it (\ie, \eqref{eqn:lin-eq-dual-mults} is underdetermined)~\cite{alizadehComplementarityNondegeneracySemidefinite1997a}.
This degeneracy is directly induced by the redundant constraints that are added to obtain tightness in the first place~\cite{dumbgenGloballyOptimalState2024,yangTEASERFastCertifiable2021a}.

In this degenerate case, the dual variables can only be found by solving an \ac{SDP}, which may be just as onerous as the direct approach above. Due to their robustness to issues such as degeneracy~\cite{alizadehPrimalDualInteriorPointMethods1998}, the \emph{de facto} approach for solving \ac{SDP} relaxations is based on \emph{interior-point methods}, which are the topic of the next section.

\subsection{Interior-Point Methods}\label{sec:bg-int-point}

In this section, we review certain ideas from the interior-point literature that are germane to the certification technique that we introduce in subsequent sections. In a nutshell, interior-point methods work by removing hard (generalized) inequality constraints and replacing them with a \emph{penalty or barrier function} in the objective that forces the inequalities to be strictly obeyed. In the context of \acp{SDP}, the inequality constraint corresponds to the cone constraint, $\bm{X}\succeq\bm{0}$, and the most commonly used barrier function is the \emph{log-determinant barrier}, $-\log\det(\bm{X})$. When applied to~\eqref{opt:SDP}, we obtain
\begin{equation}\label{opt:SDP-IPM}
	\begin{array}{rl}
		\min\limits_{\bm{X}\in\Sym^n} & \inner{\bm{C}}{\bm{X}}-\mu\log\det(\bm{X}) \\
		\st&\inner{\bm{A}_i}{\bm{X}}= b_i, \quad \forall i =1,\dots,m.
	\end{array}\tag{IPM}
\end{equation}
Here, the barrier function ensures that the eigenvalues of $\bm{X}$ are all strictly positive (\ie, $\bm{X}\in\PD^n$, the interior of the \ac{PSD} cone). The parameter $\mu\in(0,\infty)$ -- referred to as the \emph{barrier parameter} -- controls the influence of the barrier function on the optimization. In general, this parameter is decreased as the optimization progresses, according to a particular update strategy.

The trajectory, $\bm{X}(\mu)$, traced out by the (unique) solution of~\eqref{opt:SDP-IPM} as $\mu$ is varied is referred to as the \emph{central path} and will play a key role in our certification approach. As mentioned above, it can be shown that as $\mu\rightarrow0$, the central path converges to a point in the optimal (primal-dual) solution set for~\eqref{opt:SDP} that has maximal complementarity~\cite{halickaConvergenceCentralPath2002}. Note, however, that the central path does not include this limit point since it is almost never full rank, meaning that the log-determinant objective of~\eqref{opt:SDP-IPM} would be ill-defined. In practice, exact low-rank solutions are found by rounding techniques (\eg, truncating small eigenvalues).

Conversely, as $\mu\rightarrow\infty$,~\eqref{opt:SDP-IPM} becomes independent of the \ac{SDP} cost function and its solution converges to the so-called \emph{analytic center} of the feasible set~\cite{boydConvexOptimization2004a},
\begin{equation}\label{eqn:analytic-center}
    \begin{array}{rl}
    \bm{X}_{\subscr{AC}} = \arg\min\limits_{\bm{X}\in\Sym^n}-&\log\det(\bm{X})\\
    \st& ~\inner{\bm{A}_i}{\bm{X}}= b_i, \quad \forall i =1,\dots,m.
    \end{array}
\end{equation}

Interior-point methods for \acp{SDP} typically initialize the primal and dual so that they are well into the interior of the \ac{PSD} cone and choose a large initial barrier parameter accordingly~\cite{wrightPrimaldualInteriorpointMethods1997}. Though the exact details vary from method to method, iterates generally alternate between reducing the barrier parameter and solving the (non-linear) \ac{KKT} equations of \eqref{opt:SDP-IPM},
\begin{equation}\label{eqn:kkt-ipm}
    \begin{gathered}
    \quad \LinCons{\bm{X}(\mu)}= \bm{b},\\
    \bm{H}(\mu) = \bm{C} + \LinConsAdj{\bm{\lambda}(\mu)},\\
    \bm{H}(\mu)\bm{X}(\mu)=\mu\bm{I}.
    \end{gathered}
\end{equation}

For brevity, we have introduced a compact linear constraint operator and its adjoint,
\begin{equation}
    \LinCons{\bm{X}}_i = \inner{\bm{A}_i}{\bm{X}}, \quad \LinConsAdj{\bm{\lambda}} = \sum\limits_{i=1}^{m}\bm{A}_i \lambda_i,
\end{equation}
and have collected the constraint values and Lagrange multipliers into $\bm{b}\in\mathbb{R}^m$ and $\bm{\lambda}\in\mathbb{R}^m$, respectively.
Iterations continue until both the residuals of these equations and the barrier parameter are below appropriate thresholds.

We noted above that the main impediment to fast certification of a global optimum was the fact that the \ac{KKT} conditions, \eqref{eqn:SDP-KKT}, were not uniquely solvable with respect to the dual variables. Conversely, any point along the central path has dual variables, $(\bm{H}(\mu),\bm{\lambda}(\mu))$, that are uniquely defined by the \ac{KKT} conditions of~\eqref{opt:SDP-IPM}. Further, in the limit as $\mu\rightarrow 0$, these dual variables satisfy the \ac{KKT} conditions of~\eqref{opt:SDP}~\cite{klerkAspectsSemidefiniteProgramming2004}.

\section{Central Path Certifier}\label{sec:cp-cert}

The uniqueness of the dual variables along the central path motivates our main strategy in this work: we seek to find a certificate of optimality (\ie, the dual variables) by searching for a point along the central path that is `close enough' to our candidate solution. Note that our context is quite different from the standard interior-point setting; though we assume that we already have the primal solution exactly, we have no way to initialize the dual variables. As such, we specialize our method to exploit our existing knowledge of the primal solution.

\begin{figure}
    \centering
    \includegraphics[width=\columnwidth]{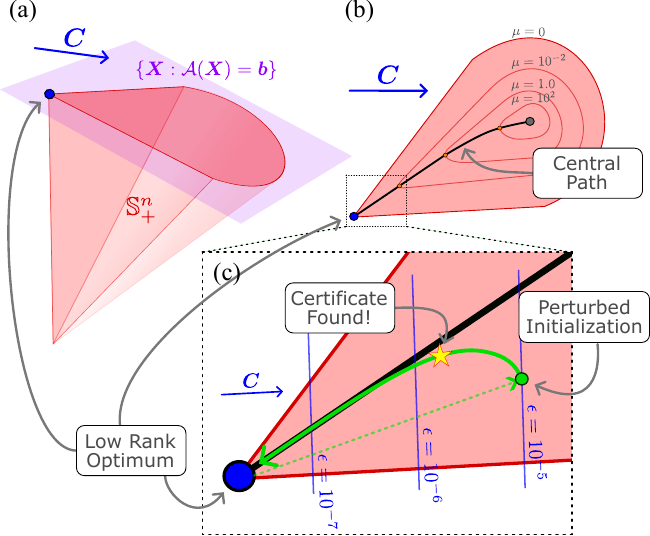}
    \caption{Core concepts of the \ac{CPCert} algorithm. (a) A three-dimensional representation of the feasible set at the intersection between the \ac{PSD} cone, $\PSD^n$, and linear constraint set, which satisfies $\mathcal{A}(\bm{X})=\bm{b}$. The linear cost function is represented with a blue arrow. (b) A two-dimensional representation of the feasible set demonstrating level sets of the barrier function (red lines) and the central path (black), which passes through the optimal solutions for~\eqref{opt:SDP-IPM} for each value of the barrier parameter, $\mu$. (c) The `round trip' of \ac{CPCert}; the low-rank solution is first perturbed towards the center of the \ac{PSD} cone (green dash), and then iterates move back towards the solution along the central path (green solid). Level sets of the cost constraint for different values of $\epsilon$ are shown as blue lines. Note that convergence is not necessarily required to find an approximate certificate (yellow star).}
    \label{fig:central-path-diag}
\end{figure}

Figure~\ref{fig:central-path-diag} presents the key concepts of the \acf{CPCert}. To obtain a certificate of optimality, we first perturb our existing (primal) solution towards the interior of the \ac{PSD} cone. We then iteratively move the perturbed solution towards the central path and converge along the central path back to our original solution. Though this `round trip' is redundant in terms of the primal variables, it will provide us with the dual variables that will certify the candidate solution.

As mentioned in Section~\ref{sec:bg-int-point}, exact convergence to the limit point of the central path cannot actually be achieved in practice. However, it turns out that we will not need to converge exactly to the low-rank solution to recover a viable global-optimality certificate. In practice, we only need to get close enough to the low-rank solution to extract an approximately valid certificate and can leverage early-stopping criteria to make our method more performant.

The following assumption is critical to our framework:
\begin{assumption}\label{assump:cand-on-path}
    We assume that if $\candx$ is globally optimal, then $\candX=\candx\candx^T$ is the limit point of the primal central path as $\mu\rightarrow0$.
\end{assumption}
\noindent The next proposition provides a relevant sufficient condition for this assumption:
\begin{proposition}\label{prop:rank-tight-suff}
    Assumption~\ref{assump:cand-on-path} is satisfied whenever the \ac{SDP} relaxation is rank tight and there is at least one constraint for which $b_i\not=0$.
\end{proposition}
\noindent This proposition is proven in Appendix~\ref{sec:app:proof-rank-tight}. For the problems studied in this paper, we note that there is always a constraint for which some $b_i$ is non-zero, and we verify rank tightness empirically for each problem.

\subsection{An Equivalent Interior-Point Problem}\label{sec:equiv-ipm}

Standard interior-point solvers initialize the barrier parameter based on the value of both the primal and the dual variables~\cite{nocedalNumericalOptimization2006b}. Since we wish to avoid explicit initialization of these dual variables, our first step is to reformulate the problem in terms of a different parameter that is more straightforward to initialize.

Taking inspiration from Lemma 1.2 of~\cite{halickaConvergenceCentralPath2002}, we note that the central path can be equivalently defined by moving the \ac{SDP} cost function from the objective into a `cost constraint',
\begin{equation}
	\begin{array}{rl}
		\min\limits_{\bm{X}\in\Sym^n} & -\log\det(\bm{X}) \\
		\st&\inner{\bm{A}_i}{\bm{X}}= b_i, \quad \forall i =1,\dots,m,\\
        & \inner{\bm{C}}{\bm{X}} = \rho(\mu),
	\end{array}
\end{equation}
where $\rho(\mu) = \inner{\bm{C}}{\bm{X}(\mu)}$ is the objective of~\eqref{opt:SDP} along the central path with respect to the barrier parameter.
At this point, we have gained nothing since the constraint is still a function of the barrier parameter. However, the right side of this constraint can be replaced by a positive constant as long as it respects the boundaries of the objective of~\eqref{opt:SDP} along the central path. We define the equivalent problem,
\begin{equation}\label{opt:CPCert}
	\begin{array}{rl}
		\min\limits_{\bm{X}\in\Sym^n} & -\log\det(\bm{X}) \\
		\st&\inner{\bm{A}_i}{\bm{X}}= b_i, \quad \forall i =1,\dots,m,\\
        & \inner{\bm{C}}{\bm{X}} =\optValSDP+\epsilon\rho_c,
	\end{array}\tag{CP}
\end{equation}
where $\rho_c = |\trace{\bm{C}}|$ is a normalization constant and $\epsilon\in(0,\bar{\epsilon})$, where
\begin{equation}\label{eqn:eps-bound}
    \optValSDP+\bar{\epsilon}\rho_c = \inner{\bm{C}}{\bm{X}_{\subscr{AC}}}
\end{equation}
corresponds to the upper bound of the objective of~\eqref{opt:SDP} at the analytic center of its feasible set.
In Appendix~\ref{sec:app:equiv-ipm}, we prove that the solution of~\eqref{opt:CPCert} is exactly the central path and that, as $\epsilon\rightarrow0$, the solution of~\eqref{opt:CPCert} converges to the optimal primal-dual variables for~\eqref{opt:SDP}.
However, in contrast to the barrier parameter, we provide an initialization for $\epsilon$ below that depends only on the initial primal variable and works well empirically.

The \ac{KKT} conditions of this problem are given by
\begin{equation}\label{eqn:cp-kkt}
    \quad \LinCons{\bm{X}}= \bm{b}(\epsilon),\quad \bm{S} = \LinConsAdj{\bm{y}}, \quad\bm{S}\bm{X}=\bm{I},
\end{equation}
where, to account for the cost constraint, we have increased the dimension of the Lagrange multiplier vector, $\bm{y}\in\mathbb{R}^{m+1}$, and have updated our operators as follows:
\begin{gather}
    \LinCons{\bm{X}}_i =\begin{cases} \inner{\bm{A}_i}{\bm{X}}, &i=1,\dots,m\\ \inner{\bm{C}}{\bm{X}},&i=m+1 \end{cases}, \\
    \LinConsAdj{\bm{y}} = \bm{C}y_{m+1}+\sum\limits_{i=1}^{m}\bm{A}_i y_i,\\
    \bm{b}(\epsilon)^T = \begin{bmatrix} \bm{b}^T & \optValSDP+\epsilon\rho_c \end{bmatrix}.
\end{gather}

As shown in Appendix~\ref{sec:app:equiv-ipm}, the \ac{KKT} equations above are equivalent to those of~\eqref{opt:SDP-IPM}, with $\mu = \frac{1}{y_{m+1}}$. Accordingly, we can also recover the corresponding central-path dual variables,
\begin{equation}\label{eqn:dual-recovery}
    \lambda_i = \frac{y_i}{y_{m+1}}~\forall i=1,\dots,m,\quad \bm{H} = \frac{1}{y_{m+1}}\bm{S}.
\end{equation}

\subsection{The CP-Cert Algorithm}\label{sec:cp-cert-algo}

Our \ac{CPCert} algorithm seeks to solve~\eqref{opt:CPCert} while simultaneously decreasing the $\epsilon$ parameter. As in most certification approaches, we assume at the outset that the optimizer of~\eqref{opt:SDP} is the candidate solution, $\candX = \candx\candx^T$, and that $\optValSDP = \inner{\bm{C}}{\candX}$.

We initialize our algorithm at the candidate solution, $\candX$, offset by the identity matrix multiplied by a small perturbation,
\begin{equation}\label{eqn:init-primal}
    \bm{X}_o = \candX + \delta\bm{I},
\end{equation}
where $\delta>0$ is the perturbation parameter. In effect, this adds perturbation $\delta$ to all eigenvalues of $\candX$, ensuring that the initialization stays close to the candidate solution and satisfies $\bm{X}_o\in\PD^n$. To match the scale of the perturbation, we select the initial parameter, $\epsilon_0 $, to be equal to this perturbation, $\epsilon_0 = \delta$. Note that this value is typically well below the bound given in~\eqref{eqn:eps-bound}.

Our choice of initial primal variable will usually violate some of the constraints, resulting in a residual in the~\ac{KKT} equations,~\eqref{eqn:cp-kkt}.
As in standard interior-point methods, our main approach is to apply Newton's method to iteratively reduce this residual, thereby obtaining a solution on the central path.
At the same time, we adaptively reduce the parameter, $\epsilon$, to converge to the limit point of the central path, which, by assumption, is the candidate solution.
Throughout this process, we recover the dual variables of~\eqref{opt:CPCert} and, at convergence, the optimal dual variables of~\eqref{opt:SDP}. Following Section~\ref{sec:bg-cert-opt}, if Assumption~\ref{assump:cand-on-path} holds, these dual variables exactly certify global optimality of $\candx$.

At each iteration, our goal is to find a step $d\bm{X}$ such that $\bm{X}^+ = \bm{X} + d\bm{X}$ reduces the residual of the \ac{KKT} equations and ensures that $\bm{X}^+\in\PD^n$. To avoid explicit initialization of the dual variables, we adopt the so-called primal-only interior-point method~\cite[Chapter 5]{klerkAspectsSemidefiniteProgramming2004} and eliminate the dual variable, $\bm{S}$, from the \ac{KKT} equations:
\begin{equation*}
    \bm{X}^{-1} = \LinConsAdj{\bm{y}},\quad  \LinCons{\bm{X}}= \bm{b}(\epsilon).
\end{equation*}
Next, we linearize these equations\footnote{For the derivation of the linearization of a matrix inverse, see~\cite{magnusMatrixDifferentialCalculus2019a}.} to obtain
\begin{equation}\label{eqn:pert-kkt}
    \begin{gathered}
    \bm{X}^{-1} - \bm{X}^{-1} d\bm{X} \bm{X}^{-1} = \LinConsAdj{\bm{y}},\\
    \LinCons{\bm{X}+d\bm{X}}= \bm{b}(\epsilon).
    \end{gathered}
\end{equation}
Pre- and post-multiplying the first equation by $\bm{X}$ and rearranging, we obtain a relation between the primal step, $d\bm{X}$, and the Lagrange multipliers,
\begin{equation}\label{eqn:primal-step}
    d\bm{X} = \bm{X} - \bm{X}\LinConsAdj{\bm{y}}\bm{X}.
\end{equation}
To find the multipliers, we substitute equation~\eqref{eqn:primal-step} into the second equation of~\eqref{eqn:pert-kkt} to obtain
\begin{equation}\label{eqn:schur-comp-sys}
    \LinCons{\bm{X}\LinConsAdj{\bm{y}}\bm{X}} = 2\LinCons{\bm{X}} - \bm{b}(\epsilon).
\end{equation}
This equation is linear in $\bm{y}$ and can be written as $\Schur\bm{y} = \bm{d}$, where
\begin{equation}
    \begin{gathered}
        \Schur_{ij} = \trace{\bm{B}_i \bm{X} \bm{B}_j \bm{X}}, \quad \bm{d}_i = 2\inner{\bm{B}_i}{\bm{X}} - \bm{b}(\epsilon)_i, \\
        \bm{B}_i = \begin{cases}
            \bm{A}_i& i\in1,\dots,m, \\
            \bm{C} & i=m+1.
        \end{cases}
    \end{gathered}
\end{equation}
This system -- known as the Schur-complement system in the optimization literature~\cite[Chapter 10.3]{boydConvexOptimization2004a} -- constitutes the main computational bottleneck of our method, and solving it efficiently is the subject of Section~\ref{sec:eff-lin-solve}.

Once the Schur-complement system is solved, the multipliers can be used to compute the primal step from~\eqref{eqn:primal-step}. The primal variable is then updated,
\begin{equation}\label{eqn:primal-update}
    \bm{X}^+ = \bm{X} + \alpha~ d\bm{X},
\end{equation}
where $\alpha$ is selected by a backtracking line search, ensuring that $\bm{X}^+\in\PD^n$. To test this condition, we use an efficient LDL factorization. Starting from an initial step size, $\alpha_0$, we repeatedly shrink $\alpha$ by a factor $\sigma_\alpha\in(0,1)$ until $\bm{X}+\alpha~d\bm{X}\in\PD^n$ or a minimum step size, $\alpha_{\min}$, is reached, in which case we terminate with failure. Algorithm~\ref{alg:line-search} summarizes this procedure.
\begin{algorithm}
    \caption{LineSearch}\label{alg:line-search}
    \begin{algorithmic}[1]
        \Require Primal iterate $\bm{X}$, primal step $d\bm{X}$
        \Function{LineSearch}{$\bm{X}$, $d\bm{X}$}
            \State $\alpha \gets \alpha_0$
            \State $d_{\min}\gets-\infty$
            \While{$d_{\min}\leq0$}
                \State $\bm{L}_x,~\bm{D}_x\gets$\textproc{LDL}$(\bm{X} + \alpha~d\bm{X})$
                \State $d_{\min}\gets\min(\diag{\bm{D}_x})$\Comment{update min diagonal}
                \If{$d_{\min}\leq0$}
                    \State $\alpha \gets \sigma_\alpha~\alpha$ \Comment{backtrack}
                    \If{$\alpha \leq \alpha_{\min}$}
                        \State \Return $\alpha_{\min}$ \Comment{reached minimum step size}
                    \EndIf
                \EndIf
            \EndWhile
            \State \Return $\alpha$
        \EndFunction
    \end{algorithmic}
\end{algorithm}

At each iteration, we apply an adaptive update law for $\epsilon$, based on the line-search parameter:
\begin{equation}\label{eqn:adapt-param}
    \begin{gathered}
        \sigma = \begin{cases}
            \sigma_{\subscr{\mbox{inc}}} & \alpha \leq \alpha_{\subscr{\mbox{inc}}}\\
            \sigma_{\subscr{\mbox{dec}}} & \alpha \geq \alpha_{\subscr{\mbox{dec}}} \\
            1 & \mbox{o.w.}
        \end{cases},\\
        \epsilon^+ = \max\{\sigma \epsilon, \epsilon_{\min}\},
    \end{gathered}
\end{equation}
where $\epsilon_{\min}$ is the smallest allowable reduction of $\epsilon$ from its initial value. Default parameter values are provided in Table~\ref{tab:parameters} unless otherwise specified.
Intuitively, if the step size is restricted because the iterate is too close to the boundary of the \ac{PSD} cone, then $\epsilon$ is increased to target a central-path point that is farther from the boundary. Conversely, if large steps are being made, then it is safe to reduce $\epsilon$.

\subsection{Stopping Conditions}

As mentioned above, we can recover a global-optimality certificate as soon as the iterates are close enough to a point on the central path that is near the candidate optimum. At every iteration, we use the current multipliers to construct a candidate certificate matrix based on~\eqref{eqn:dual-recovery},
\begin{equation}
    \hat{\bm{H}} = \frac{1}{y_{m+1}}\LinConsAdj{\bm{y}} = \bm{C} + \sum\limits_{i=1}^m \bm{A}_i \frac{y_i}{y_{m+1}}.
\end{equation}
If this matrix satisfies the \ac{KKT} conditions in~\eqref{eqn:SDP-KKT} for $\candX$, then we have successfully obtained a certificate of optimality for $\candx$. We first check the complementarity condition to a numerical tolerance, $\tau_c$:
\begin{equation}
    \vert\inner{\hat{\bm{H}}}{\candX}\vert = \vert \candx^T \hat{\bm{H}} \candx \vert \leq \tau_c.
\end{equation}
If this relation is true, we next test if $\hat{\bm{H}} + \bm{I}\tau_p \in \PSD^n$, for small tolerance $\tau_p>0$, by attempting a Cholesky decomposition~\cite{papaliaCertifiablyCorrectRangeAided2024}. Note that, by construction, \ac{CPCert} can only certify optimality if it finds a set of Lagrange multipliers that satisfy the~\ac{KKT} conditions discussed in Section~\ref{sec:bg-cert-opt}.

On the other hand, if either the \ac{SDP} relaxation is not rank tight or the candidate solution is not globally optimal, the primal iterate, $\bm{X}$, will diverge from the candidate solution, $\candX$. To robustly detect this condition, we use the angle between the symmetric matrices, $\candx\candx^T$ and $\bm{X}$, which is well defined via the Frobenius inner product~\cite{andreaniGeometricalStructureSymmetric2013}:
\begin{equation}
    \theta(\bm{X},\candx) = \arccos \left(\frac{\candx^T\bm{X}\candx}{\|\bm{X}\| ~\|\candx\candx^T\|}\right).
\end{equation}
If the absolute value of the angle exceeds a threshold, $\theta_{\max}>0$, we terminate the algorithm and claim that it has failed to certify the solution. This criterion serves as an early-stopping condition when either the candidate is not globally optimal or the relaxation is not rank tight.

We stop iterating when the magnitude of the primal step has dropped below a given tolerance,
\begin{equation}
    \| d\bm{X} \| \leq \tau_{\mbox{\scriptsize step}},
\end{equation}
and $\epsilon$ is at a predefined minimum value, $\epsilon = \epsilon_{\min}$.\footnote{In our implementation, we define $\epsilon_{\min}$ relative to the initial value of $\epsilon$; that is, we define a ratio $\tilde{\epsilon}_{\min}$, so that $\epsilon_{\min} = \tilde{\epsilon}_{\min}\epsilon_0$.} We also limit the maximum number of iterations to $K_{\max}$. Termination in either of these cases leads to a failure to certify.

Algorithm~\ref{alg:cp-cert} summarizes the full \ac{CPCert} algorithm.
\begin{algorithm}
    \caption{Central Path Certifier (CP-Cert)}\label{alg:cp-cert}
    \begin{algorithmic}[1]
        \Require Candidate solution $\candx$
        \Ensure Certificate $\hat{\bm{H}}$ verifying global optimality of $\candx$, or failure to certify
        \State $\bm{X} \gets \candx\candx^T + \delta\bm{I}$ \Comment{initialize primal}
        \State $\epsilon\gets\epsilon_0$ \Comment{initialize parameter}
        \State $\bm{L}_p,\bm{D}_p\gets $\textproc{LDL}$(\Precond_{\subscr{\mbox{aug}}})$ \Comment{factorize system~\eqref{eqn:aug-schur-sys}}
        \For{$k = 1, \dots, K_{\max}$}
            \State \textit{Primal Update:}
            \State $\bm{y}\gets $\textproc{PrecondConjGrad}$(\bm{X},\bm{y},\epsilon,\bm{L}_p,\bm{D}_p)$ \Comment{solve~\eqref{eqn:schur-comp-sys}}
            \State $\bm{S}\gets\LinConsAdj{\bm{y}}$ \Comment{store dual}
            \State  $d\bm{X}\gets \bm{X} - \bm{X}\bm{S}\bm{X}$ \Comment{compute step via~\eqref{eqn:primal-step}}
            \State $\alpha\gets$\textproc{LineSearch}$(\bm{X},d\bm{X})$ \Comment{compute step size}
            \State $\bm{X}\gets\bm{X} + \alpha~d\bm{X}$ \Comment{update solution}
            \State $\sigma\gets\begin{cases}
            \sigma_{\subscr{\mbox{inc}}} & \alpha \leq \alpha_{\subscr{\mbox{inc}}}\\
            \sigma_{\subscr{\mbox{dec}}} & \alpha \geq \alpha_{\subscr{\mbox{dec}}} \\
            1 & \mbox{o.w.}
            \end{cases}$ \Comment{adapt. via~\eqref{eqn:adapt-param}}
            \State $\epsilon\gets\max\{\sigma\epsilon,\epsilon_{\min}\}$ \Comment{update param.}
            \State \textit{Stopping criteria:}
            \State $\hat{\bm{H}}\gets\bm{S} / y_{m+1}$ \Comment{get candidate certificate via~\eqref{eqn:dual-recovery}}
            \If{$\vert\candx^T\hat{\bm{H}}\candx\vert \le \tau_c$}
                \If {$\mbox{Cholesky}(\hat{\bm{H}} + \tau_p\bm{I})$ succeeds}
                    \State \Return $\hat{\bm{H}}$  \Comment{success: certificate found}
                \EndIf
            \EndIf
            \If{$|\theta(\bm{X}, \candx)| \ge \theta_{\max}$}
                \State \Return failure \Comment{iterate diverged from $\candx$}
            \EndIf
            \If{$\|d\bm{X}\| \le \tau_{\mbox{\scriptsize step}}$ \textbf{and} $\epsilon = \epsilon_{\min}$}
                \State \Return failure \Comment{converged without certifying}
            \EndIf
        \EndFor
        \State \Return failure \Comment{max iterations reached}
    \end{algorithmic}
\end{algorithm}

\subsection{Solving the Linear System Efficiently}\label{sec:eff-lin-solve}

Many state-of-the-art interior-point solvers directly factorize and solve the full set of \ac{KKT} conditions given in~\eqref{eqn:kkt-ipm} at each iteration.
Though this \emph{primal-dual} approach is highly robust~\cite{alizadehPrimalDualInteriorPointMethods1998}, its runtime complexity is intractable for many \ac{SDP} relaxations in robotics, which are often high-dimensional and have a large number of constraints~\cite{yangCertifiablyOptimalOutlierRobust2023a}.
Instead, we take the approach of recent scalable interior-point solvers~\cite{habibiLoraineInteriorpointSolver2024,zhangModifiedInteriorpointMethod2017,tohSolvingLargeScale2004} and solve the Schur-complement system,~\eqref{eqn:schur-comp-sys}, indirectly via the preconditioned conjugate-gradient method~\cite{nocedalNumericalOptimization2006b}.

Conjugate gradient is an iterative method that only requires a means to form the matrix-vector products, $\Schur\bm{y}$. As such, this method enjoys increased scalability compared to direct methods where the Schur-complement system needs to be explicitly formed and factorized.

Given the $k$-th iterate, $\bm{y}^{(k)}$, we efficiently form the vector, $\bm{z}^{(k)}=\Schur\bm{y}^{(k)}$, in three steps via~\eqref{eqn:schur-comp-sys}:
\begin{equation*}
\begin{gathered}
    \bm{S}^{(k)}=\LinConsAdj{\bm{y}^{(k)}} = \sum\limits_{i=1}^{m+1} \bm{B}_i y^{(k)}_i, \\
    \bm{Y}^{(k)} = \bm{X} \bm{S}^{(k)} \bm{X}, \\
    z^{(k)}_i = \trace{\bm{B}_i \bm{Y}^{(k)}}, ~\forall~i=1,\dots,m+1.
\end{gathered}
\end{equation*}
This operation has linear complexity in $m$ and allows us to fully leverage the sparsity of the constraint matrices as well as efficient parallelization of the final step. To reduce the number of iterations until convergence, we initialize the conjugate-gradient method with multipliers from the previous iteration of the main algorithm.\footnote{On the first iteration, we initialize with zeros.}

An unfortunate downside of the conjugate-gradient method is that it is highly susceptible to instability when the matrix, $\Schur$, becomes ill-conditioned, which is inevitably the case as interior-point methods converge to their solution~\cite{habibiLoraineInteriorpointSolver2024,wrightIllConditioningComputationalError1998}. To see why this is the case in our setting, consider the following alternate form of the Schur-complement matrix:
\begin{equation}
    \Schur = \VecCons^T(\bm{X}\otimes\bm{X})\VecCons,\quad [\VecCons]_i = \vect{\bm{B}_i},
\end{equation}
where $\VecCons$ is the matrix formed by concatenating the vectorized $\bm{B}_i$ matrices. Though $\VecCons$ typically has full column rank,\footnote{This is true under the standard constraint qualifications for \acp{SDP} coupled with the fact that $\bm{C}$ is typically linearly independent of the constraint matrices $\bm{A}_i$.} $\bm{X}$ converges to a low-rank solution and its resulting ill-conditioning is amplified by the Kronecker product. As such, we need an effective way to precondition the linear system.

\subsection{An Effective Preconditioner}

Several works deal with the ill-conditioning above by applying the \emph{preconditioned conjugate-gradient} method~\cite{habibiLoraineInteriorpointSolver2024,tohSolvingLargeScale2004}. An effective preconditioner, $\Precond$, is a symmetric, positive-definite operator that closely approximates the spectrum of the matrix that it is preconditioning~\cite{nocedalNumericalOptimization2006b}. This allows one to apply the conjugate-gradient method to the preconditioned system,
\begin{equation}
    \Precond^{-1} \Schur \bm{y} = \Precond^{-1}\bm{d},
\end{equation}
whose condition number is significantly lower than that of the original system, since $\Precond^{-1} \Schur \approx \bm{I}$.

Though finding an effective preconditioner for the Schur-complement matrix remains an open challenge, several recent works have resorted to finding effective \emph{local} preconditioners that are periodically recomputed as the primal variable converges~\cite{habibiLoraineInteriorpointSolver2024,zhangModifiedInteriorpointMethod2017}.
Recomputing this preconditioner often leads to another computational bottleneck for these approaches.

The key insight in our context is that our method is designed to remain close to a target solution, $\candX$. Consequently, we can compute and factorize the preconditioner once at the beginning of execution, and reuse it for all iterations.

We note that applying a standard diagonal (or Jacobi) preconditioner did not provide sufficient stability to apply \ac{PCG} effectively in our test problems. Instead, we define our preconditioner as the Schur-complement matrix evaluated at a perturbed version of our candidate solution, $\tilde{\bm{X}} = \candx\candx^T + \tau\bm{I}$ (similar to~\cite{zhangModifiedInteriorpointMethod2017}). Concretely,
\begin{equation}\label{eqn:precond-1}
    \Precond = \VecCons^T(\tilde{\bm{X}}\otimes\tilde{\bm{X}})\VecCons.
\end{equation}
We show in Appendix~\ref{sec:app:precond} that the preconditioner can be equivalently expressed as
\begin{equation}\label{precond-2}
    \Precond = \tau^2\VecCons^T\VecCons + \bm{V}\bm{V}^T,
\end{equation}
where $\bm{V}\in \mathbb{R}^{(m+1)\times (n+1)}$ is given by
\begin{equation}
    \bm{V} = \begin{bmatrix}
        \VecCons^T(\candx\otimes\candx) & \sqrt{2\tau} \VecCons^T(\candx\otimes\bm{I})
    \end{bmatrix}.
\end{equation}
It was noted in~\cite{zhangModifiedInteriorpointMethod2017} that a standard `implementation trick' for improved stability when applying the inverse of this matrix to a vector, $\bm{w}$, is to instead solve the following sparse augmented linear system:\footnote{Note that the original system can be recovered from the augmented system by eliminating $\bm{z}_2$, \ie, taking the Schur complement.}
\begin{equation}\label{eqn:aug-schur-sys}
    \Precond_{\subscr{\mbox{aug}}}\bm{z}=\begin{bmatrix}
        \tau^2 \VecCons^T\VecCons & \tau \bm{V} \\ \tau \bm{V}^T & -\tau^2\bm{I}
    \end{bmatrix} \begin{bmatrix}\bm{z}_1 \\ \bm{z}_2\end{bmatrix}= \begin{bmatrix}\bm{w}\\ \bm{0}\end{bmatrix}.
\end{equation}
The inverse is then given by $\bm{z}_1$. Though the dimension of this matrix is increased by $n+1$, it is highly sparse in practice and can be efficiently factorized using a sparse LDL decomposition. In contrast to other interior-point methods~\cite{habibiLoraineInteriorpointSolver2024,zhangModifiedInteriorpointMethod2017}, we compute this factorization once at the beginning of the algorithm and reuse the (sparse) factors at each iteration. A good choice of $\tau$ when constructing the preconditioner is to set it equal to the initial primal perturbation, $\delta$.

\subsection{Implementation}

We implement \ac{CPCert} in \texttt{C++}, making extensive use of the \texttt{Eigen} library and exploiting sparsity and parallelism wherever possible.\ifanonymous
 Code will be made available should this paper be accepted.
\else
\footnote{The code implementing \ac{CPCert} is available at \url{https://github.com/utiasASRL/cp-cert-core}, including python bindings.}
\fi
\section{Application to Outlier-Robust Stereo-Camera Pose Registration}\label{sec:applications}

In this section, we introduce two key applications of our approach, which can be combined to yield a certifiable and robust pose-registration pipeline for pointclouds derived from stereo-camera data.

\begin{figure}
    \centering
    \includegraphics[width=\columnwidth]{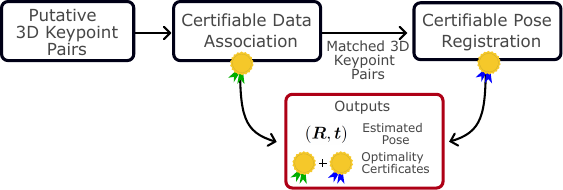}
    \caption{A simplified representation of our approach to certification of outlier-robust pointcloud registration. We assume that pointclouds and putative correspondences are generated by some preprocessing method. Data association and pointcloud registration are then performed and certified in two separate steps, resulting in an estimated pose and the optimality certificate from each component. A more detailed pipeline diagram is shown in Figure~\ref{fig:pipeline}.}
    \label{fig:pipeline-simplified}
\end{figure}

Outlier-robust pose registration establishes a relative transformation between two distinct robot poses based on sensor data collected at each pose. These sets of data are often preprocessed into a pair of distinct pointclouds, which are then used to estimate the relative pose in two key steps: data association and pointcloud registration. The former can be framed as a combinatorial optimization problem that seeks to identify correspondences between pointclouds, while the latter uses these associations to find a transformation that optimally aligns the pointclouds in three-dimensional space. Both subproblems are susceptible to convergence to spurious local minima.

To be robust, our pipeline must be able to deal with incorrect associations in addition to corruption due to measurement noise. To be certifiable, the pipeline must verify that it has converged to a globally optimal solution.

To our knowledge, existing robust and certifiable registration pipelines do not seek to certify the optimality of the data-association problem directly~\cite{yangCertifiablyOptimalOutlierRobust2023a,yangTEASERFastCertifiable2021a,korotkineGloballyOptimalDataAssociationFree2025b}. Instead, they allow outlier measurements to appear in the pointcloud-registration optimization problem and mitigate their influence using a robust-cost, M-estimation technique~\cite{mactavishAllCostsComparison2015,barfootStateEstimationRobotics2024}. Applying the so-called Black--Rangarajan duality~\cite{blackUnificationLineProcesses1996}, these works convert the registration problem into a \ac{QCQP} for which a tight \ac{SDP} relaxation exists. However, the resulting formulation requires a large number of constraints and becomes intractably large for typical problems in robotics. For example, for problems of about 200 correspondences, this approach takes about 12 s on average (cf. Figure 9 of~\cite{yangTEASERFastCertifiable2021a}).

In this section, we take a different approach; treating data association as a combinatorial optimization problem, we first find and certify global optimality of a putative set of correspondences.
Since these correspondences are certified to be correct, we then avoid the complexity of making the pose estimation robust to outliers. As a result, we can apply standard pointcloud-regression techniques, which have \ac{SDP} relaxations that can be certified in milliseconds.
In particular, we adopt the \emph{matrix-weighted} approach described in~\cite{holmesSemidefiniteRelaxationsMatrixWeighted2024a} to solve and certify the regression problem, which allows one to properly account for the increased depth uncertainty in pointclouds generated from stereo-camera data.

It is important to note that this two-step certification process is not equivalent to certifying the entire outlier-robust pointcloud-registration problem outright, as is done in~\cite{yangTEASERFastCertifiable2021a}. However, we maintain that our approach is more tractable, while still providing strong guarantees of solution quality.

\subsection{Certifying Pointcloud Data Association}\label{sec:cert-data-assoc}

We wish to find a set of correspondences, $\clique$, between two 3D pointclouds, $\SrcSet = \{\bm{p}_i\}_{i=1}^N$ and $\TrgSet=\{\bm{q}_j\}_{j=1}^N$.
To obtain global optimality guarantees for this set, we adopt a recent graph-theoretic paradigm for data association~\cite{shiROBINGraphTheoreticApproach2021b,luskCLIPPERGraphTheoreticFramework2021,yangTEASERFastCertifiable2021a}. We construct a \emph{consistency graph}, in which nodes correspond to putative correspondences and weighted edges encode geometric consistency between these correspondences~\cite{luskCLIPPERRobustData2024}. Given correspondences $(i,k)$ and $(j,l)$, edges are only constructed if
\begin{equation}
    \vert\delta\vert = \vert \| \bm{p}_i - \bm{p}_j\|_2 - \| \bm{q}_k - \bm{q}_l\|_2 \vert < \epsilon,
\end{equation}
where $\bm{p}_i$ and $ \bm{p}_j$ are points in the first pointcloud, $\bm{q}_k$ and $ \bm{q}_l$ are points in the second pointcloud, and $\epsilon$ is a user-defined threshold. Weights are then assigned to valid edges according to a \emph{score function},
\begin{equation}
    s(\delta) = \exp\left(-\frac{1}{2}\frac{\delta^2}{\sigma^2}\right),
\end{equation}
where $\sigma$ controls the score assigned to each association (see~\cite{luskCLIPPERRobustData2024}). The construction of the consistency graph is demonstrated in Figure~\ref{fig:consistency-graph}.

\begin{figure}[t]
    \centering
    \includegraphics[width=\columnwidth]{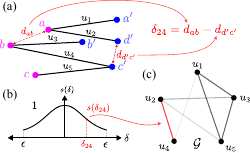}
    \caption{Construction of the consistency graph from a set of correspondences, $\{u_i\}$, between pointclouds. (a) shows the pointclouds and putative correspondences. For any two correspondences, distances between the associated points in each pointcloud (\eg, $d_{ab}$ and $d_{d'c'}$) are computed and compared across pointclouds to get a measure, $\delta$, of consistency. In (b), the score function is then applied to $\delta$, yielding the score, $s(\delta)$, which then serves as the edge weight between the two correspondences (nodes) in the consistency graph shown in (c). The best set of correspondences is then the most heavily weighted clique of the consistency graph.}
    \label{fig:consistency-graph}
\end{figure}

In this framework, the cliques (\ie, complete subgraphs) of this graph represent sets of associations that are all mutually geometrically consistent. The problem of finding the best mutually consistent clique can then be formulated as the \emph{densest edge-weighted clique (DEWC)} problem~\cite{luskCLIPPERRobustData2024},
\begin{equation}\label{opt:max-clique}
\begin{array}{rl}
    \max\limits_{\bm{u} \in \left\{ 0,1\right\}^n} & \frac{\bm{u}^T\Affinity\bm{u}}{\bm{u}^T\bm{u}} \\
    \st & u_i u_j = 0~\mbox{if}~\Affinity_{ij}=0,~\forall~i,j.
\end{array}\tag{DEWC}
\end{equation}
Here, $u_i$ is a binary variable encoding whether node $i$ is in the optimal clique, and $\Affinity$ -- referred to as the \emph{affinity matrix} -- is the (weighted) adjacency matrix of the graph with unit diagonal. The bilinear constraints on the entries of $\bm{u}$ encode the fact that the solution must be a clique of the graph.\footnote{To see why this is true, consider the fact that if $i$ and $j$ are both in a clique, then there must be an edge between them and $\Affinity_{ij}>0$. The constraint enforces the converse of this fact.}

Several works also study the \emph{unweighted} version of this problem (all edges assigned a weight of one), where the optimization reduces to the so-called maximum clique problem~\cite{shiROBINGraphTheoreticApproach2021b,fathianCLIPPERFastMaximal2024a}.\footnote{As shown in~\cite{luskCLIPPERRobustData2024}, the weighted version of the problem is more likely to have a unique optimum than the unweighted case. This is because the edge-weighting plays a ``tie-breaking'' role, which helps differentiate the clique that is the most geometrically consistent.}

To facilitate a fast local solver for this problem,~\cite{luskCLIPPERRobustData2024} introduces the \emph{maximum spectral radius clique} problem by relaxing the binary domain of this problem to the positive reals,
\begin{equation}\label{opt:MSRC}
\begin{array}{rl}
    \max\limits_{\bm{x} \in \mathbb{R}^n_+} & \bm{x}^T\Affinity\bm{x} \\
    \st & x_i x_j = 0~\mbox{if}~\Affinity_{ij}=0,~\forall~i,j, \\
    & \left\|\bm{x}\right\|^2_2 = 1.
\end{array}\tag{MSRC}
\end{equation}
\begin{remark}
    In~\cite{luskCLIPPERGraphTheoreticFramework2021} and related works, the final constraint in~\eqref{opt:MSRC} is written as an inequality, $\left\|\bm{x}\right\|^2_2 \leq 1$. We show in Appendix~\ref{sec:app:equiv-msrc} that this inequality is binding for any optimum as long as $\Affinity$ is not identically zero. We have therefore replaced it with an equality.
\end{remark}
\begin{remark}
    Similar to~\eqref{opt:max-clique}, the first constraint of~\eqref{opt:MSRC} enforces that the support of $\bm{x}$ is a clique of the data-association graph.
\end{remark}
Though the optima of this (non-convex) relaxation do not, in general, correspond to those of \eqref{opt:max-clique}, the authors of~\cite{luskCLIPPERRobustData2024} show that it is still remarkably effective for finding inlier sets. Since the solution is a continuous variable, the discrete inlier set is found via a rounding procedure~\cite{luskCLIPPERRobustData2024}.

The authors of~\cite{luskCLIPPERRobustData2024} also explore a convex relaxation of~\eqref{opt:MSRC} that is empirically tight, but largely intractable due to its dependence on the so-called \emph{doubly non-negative cone},
\begin{equation}\label{eqn:doubly-non-neg-cone}
    \mathbb{D}^n=\{\bm{X}\in\PSD^n :~ \bm{X}_{ij} \geq 0,~\forall~i,j\},
\end{equation}
which applies constraints to every entry of $\bm{X}$.

Interestingly, when applied to unweighted graphs, this relaxation is part of a hierarchy of relaxations for the maximum clique problem that have been well studied by the optimization community. Moreover, it turns out that it can be further relaxed to the so-called Lov\'{a}sz-theta problem, which forgoes the expensive inequality constraints of the doubly non-negative cone, yet remains tight in many cases~\cite{dukanovicSemidefiniteProgrammingRelaxations2007}.
Inspired by this idea, we consider the following relaxation of~\eqref{opt:MSRC}:
\begin{equation}\label{opt:lovasz-theta}
\begin{array}{rl}
    \max\limits_{\bm{X} \in \PSD^n} & \inner{\Affinity}{\bm{X}} \\
    \st & \bm{X}_{ij} = 0~\mbox{if}~\Affinity_{ij}=0,~\forall~i,j,\\
    & \trace{\bm{X}} = 1.
\end{array}\tag{LT}
\end{equation}
Especially when $\Affinity$ is dense (\ie, there are fewer outliers), this relaxation significantly reduces the number of constraints required.
Noting that all of the constraints are now linear equalities, it is straightforward to show that~\eqref{opt:lovasz-theta} can be expressed in the form of~\eqref{opt:SDP} (letting $\bm{C}=-\Affinity$).
Moreover, as we will show, the relaxation is empirically rank tight (\ie, yields rank-one solutions) when score-function parameters are set appropriately (see Section~\ref{sec:noise-param-exp}).
As such, we can apply our central-path certifier to verify global optimality of any candidate solution to~\eqref{opt:MSRC} found via fast, local optimization methods. In particular, we can use the CLIPPER algorithm proposed in~\cite{luskCLIPPERRobustData2024} to find candidate solutions to \eqref{opt:MSRC} and certify optimality by applying \ac{CPCert} to~\eqref{opt:lovasz-theta}.

\subsection{Adapting to Discrete Candidate Solutions}\label{sec:disc-adapter}

We can apply \ac{CPCert} to~\eqref{opt:lovasz-theta} to certify global optimality of a \emph{continuous} solution, $\candx$, to~\eqref{opt:MSRC}. However, there are many techniques that provide \emph{discrete} solutions to data-association problems (\eg, \ac{RANSAC}~\cite{fischlerRandomSampleConsensus1981} or \ac{PMC}~\cite{rossiParallelMaximumClique2015,shiROBINGraphTheoreticApproach2021b}). We would also like to be able to efficiently verify global optimality of these candidates, but require a means to map a discrete set to a candidate solution for~\eqref{opt:MSRC}.

To this end, we have the following theorem, which relates a solution of~\eqref{opt:MSRC} to its support:
\begin{theorem}\label{thm:supp-to-soln}
    Let $\bm{x}^*$ be a globally optimal solution to~\eqref{opt:MSRC} and let $\clique=\supp{\bm{x}^*}$ be the associated clique of the data-association graph. Let $\Affinity_{\clique}$ be the principal submatrix of the affinity matrix, $\Affinity$, found by keeping the rows and columns corresponding to the indices in $\clique$. Then $\bm{x}^*$ is a vector of zeros except on the indices of $\clique$, where it is equal to the leading eigenvector of $\Affinity_{\clique}$.
\end{theorem}
\begin{proof}
    We can restrict the feasible set of~\eqref{opt:MSRC} to the support (set of non-zero elements) of the globally optimal solution, $\clique$,
    \begin{equation*}
    \begin{array}{rl}
        \max\limits_{\bm{x}_{\clique} \in \mathbb{R}^{\card{\clique}}_+} & \bm{x}_{\clique}^T\Affinity_{\clique}\bm{x}_{\clique} \\
        \st & \left\|\bm{x}_{\clique}\right\|^2_2 = 1,
    \end{array}
    \end{equation*}
    where $\bm{x}_{\clique}$ is the vector of elements of $\bm{x}$ indexed by $\clique$.
    The clique constraints have disappeared since they are automatically satisfied by the restriction ($\clique$ is a clique). We are left with a maximum-eigenvalue optimization problem over the positive reals. However, since $\Affinity_{\clique}$ is a positive matrix, its leading eigenvector has all positive values by the Perron--Frobenius theorem~\cite{hornMatrixAnalysis2012}. This eigenvector is therefore the solution to the restricted problem and the solution to~\eqref{opt:MSRC} can be found by inserting zeros into this vector at the non-clique indices.
\end{proof}
To obtain a candidate optimum for~\eqref{opt:MSRC}, we start by checking if the discrete set is a clique of the data-association graph. If not, we conclude that there must be an outlier that is not geometrically consistent with the other members of the set. Otherwise, we assume that the set is the support of the optimal solution and apply Theorem~\ref{thm:supp-to-soln} to find a candidate solution.
This can be done efficiently by restricting the affinity matrix and applying power iterations to find its leading eigenvector~\cite{golubMatrixComputations2013a}.
Algorithm~\ref{alg:disc-to-cont} summarizes this procedure.
\begin{algorithm}
    \caption{DiscreteToContinuous}\label{alg:disc-to-cont}
    \begin{algorithmic}[1]
        \Require Discrete candidate set $\bm{u}$, affinity matrix $\Affinity$
        \Function{DiscreteToContinuous}{$\bm{u}$, $\Affinity$}
            \State $\Affinity_{\bm{u}} \gets$ $\Affinity[\bm{u},\bm{u}]$ \Comment{$\Affinity$ as indexed by $\bm{u}$}
            \If{$\exists~i,j : \left(\Affinity_{\bm{u}}\right)_{ij} = 0$} \Comment{Check for non-edges}
                \State \Return failure \Comment{$\bm{u}$ contains an outlier}
            \EndIf
            \State $\bm{v} \gets \mbox{PowerIteration}(\Affinity_{\bm{u}})$ \Comment{max eigenvector via~\cite{golubMatrixComputations2013a}}
            \State $\bm{x} \gets \bm{0}$
            \State $\bm{x}[\bm{u}] \gets \bm{v}$ \Comment{embed into full support vector}
            \State \Return $\bm{x}$
        \EndFunction
    \end{algorithmic}
\end{algorithm}

\subsection{Certifiable Stereo-Pointcloud Registration}\label{sec:cert-reg}

In this section, we assume that the set of correspondences, $\clique$, has been found and certified via the method presented in the previous sections. We collect matched pairs into a new set, $\MatchSet=\{(\bm{p}_i, \bm{q}_j)~\forall~(i,j)\in\clique\}$. Since all outlier associations have been removed, we assume that the matched pointclouds are related by a rigid transformation and corrupted with zero-mean, Gaussian noise:
\begin{equation}\label{eqn:cloud-noise}
    \bm{p}_i = \bm{R}\bm{q}_j + \bm{t} + \bm{\eta}_i,\quad \bm{\eta}_i\sim\mathcal{N}(\bm{0},\bm{\Sigma}_i),\quad \forall (\bm{p}_i,\bm{q}_j) \in \MatchSet,
\end{equation}
where $(\bm{R},\bm{t})$ are the rotation and translation components of the rigid transformation, respectively, and $\bm{\Sigma}_i\in\PD^3$ is the covariance associated with the matched pair.

Crucially, we do not assume that these covariances are isotropic. It has been shown that, especially when pointclouds are derived from stereo-camera data, accurate relative-pose estimates depend on proper accounting of the increased measurement uncertainty in the camera's depth direction~\cite{matthiesErrorModelingStereo1987a}.

The maximum-likelihood estimate of the poses is given by the solution to the following matrix-weighted, least-squares optimization problem~\cite{barfootStateEstimationRobotics2024}:
\begin{equation}\label{opt:mat-weight-reg}
		\min\limits_{\bm{R}\in\SO{3},\bm{t}\in\mathbb{R}^3} \sum\limits_{(i,j)\in\mathcal{M}} \left\|\bm{p}_i - \left(\bm{R}\bm{q}_j + \bm{t}\right)\right\|^2_{\bm{\Sigma}^{-1}_{i}},
\end{equation}
where $\| \bm{e}\|^2_{\bm{W}} = \bm{e}^T\bm{W}\bm{e}$ denotes the squared Mahalanobis distance. This optimization problem can be re-expressed as a \ac{QCQP} and has an \ac{SDP} relaxation that is rank tight even at high noise levels. However, the relaxation requires the use of redundant constraints that make the \ac{SDP} solution degenerate (\ie, it has a non-unique dual solution). Details of the explicit form of the relaxation are provided in~\cite{holmesSemidefiniteRelaxationsMatrixWeighted2024a}.

Though~\cite{holmesSemidefiniteRelaxationsMatrixWeighted2024a} uses standard interior-point methods to solve the relaxation itself, we will use the relaxation to certify solutions found via fast, local-optimization methods. In particular, we implement a local solver using the \ac{GTSAM} library~\cite{gtsam} and certify global optimality using \ac{CPCert}.

\section{Experiments}\label{sec:experiments}

We now investigate the performance of \ac{CPCert} for certifying solutions to the data-association and matrix-weighted pointcloud-registration problems discussed above. In Section~\ref{sec:sim-exp}, we first apply \ac{CPCert} to simulated examples to verify its performance under a controlled set of parameters (\eg, noise, outlier ratios, and problem dimension). This will also afford us an opportunity to study the circumstances under which the CLIPPER algorithm can be expected to converge to the globally optimal solution of~\eqref{opt:MSRC}. In Section~\ref{sec:stereo-pipeline}, we then validate the performance of \ac{CPCert} on a stereo-vision pipeline applied to a subset of the EuRoC dataset~\cite{burriEuRoCMicroAerial2016}. The parameters used for \ac{CPCert} are provided in Table~\ref{tab:parameters} for each of the problems considered.

We assume that a pose estimate, $\bm{T}_{\subscr{est}}\in\SE{3}$, can be expressed as the corresponding ground-truth pose, $\bm{T}_{\subscr{gt}}\in\SE{3}$, perturbed by an error defined in the Lie algebra of $\SE{3}$,
\begin{equation}
    \bm{T}_{\subscr{est}} = \bm{T}_{\subscr{gt}} \Exp(\bm{\xi}_{\subscr{err}}),
\end{equation}
where $\Exp(\cdot)$ is the so-called capitalized exponential map of the $\SE{3}$ Lie group and $\bm{\xi}_{\subscr{err}}\in\mathbb{R}^6$ is a vector representation of the estimation error defined in the $\SE{3}$ Lie algebra (see~\cite{solaMicroLieTheory2021} for formal definitions of these concepts). Accordingly, we define the translation and rotation estimation errors as the following (scalar) quantities:
\begin{equation}
        e_t = \|\bm{\xi}_{\subscr{trans}}\|_2,\quad e_r = \|\bm{\xi}_{\subscr{rot}}\|_2,
\end{equation}
where $\bm{\xi}_{\subscr{trans}}$ and $\bm{\xi}_{\subscr{rot}}$ are the translational and rotational components of the error, respectively.

All experiments were run on an Intel(R) Xeon(R) CPU E5-2698 running at 2.20 GHz.\ifanonymous
\else
\footnote{Code for the experiments presented in this section is available at \url{https://github.com/utiasASRL/cp-cert-experiments}.}
\fi
\section{Application to Outlier-Robust Stereo-Camera Pose Registration}\label{sec:applications}

\subsection{Stanford Bunny Experiments}\label{sec:sim-exp}

We replicate the simulated experiment setup used in~\cite{luskCLIPPERRobustData2024} to evaluate the CLIPPER algorithm; the Stanford Bunny dataset~\cite{turkZipperedPolygonMeshes1994} was used to generate two pointclouds along with a set of putative correspondences of varying size and fraction of outliers. One of the pointclouds was corrupted with zero-mean, bounded, isotropic, Gaussian noise, $\bm{\eta} \sim\mathcal{N}(\bm{0}, \gamma^2\bm{I})$, $\|\bm{\eta}\|_2 \leq \beta$, where $\gamma=0.01$ m and $\beta = 5.54\gamma$. The latter pointcloud was also transformed by a randomly generated rigid-body transformation.

In all cases, we compare \ac{CPCert} to directly solving the~\ac{SDP} relaxation via Mosek~\cite{apsMOSEKOptimizerAPI2024}, a state-of-the-art interior-point solver. We remind the reader that, for the degenerate \acp{SDP} under consideration, this direct solve is currently the most performant alternative for certification. To provide a fair comparison, both \ac{CPCert} and Mosek were run using sparse matrices with parallelization enabled (using 20 threads).

\subsubsection{Data Association Local Minimum Experiment}

We first demonstrate an initial experiment that informs our understanding of the circumstances under which CLIPPER can converge to local versus global minima. These experiments revealed that CLIPPER is quite robust to so-called adversarial scenarios. We demonstrate such an adversarial scenario in Figure~\ref{fig:adv-example}, in which all of the outlier associations form a second, geometrically consistent clique. Since CLIPPER is a local method, it can be induced to converge to the incorrect solution if initialized poorly. However, we observed that the default initialization generally selects the correct clique in these cases. Regardless, \ac{CPCert} accurately certifies the globally optimal solution.

\begin{figure*}[t]
    \centering
    \includegraphics[width=\textwidth]{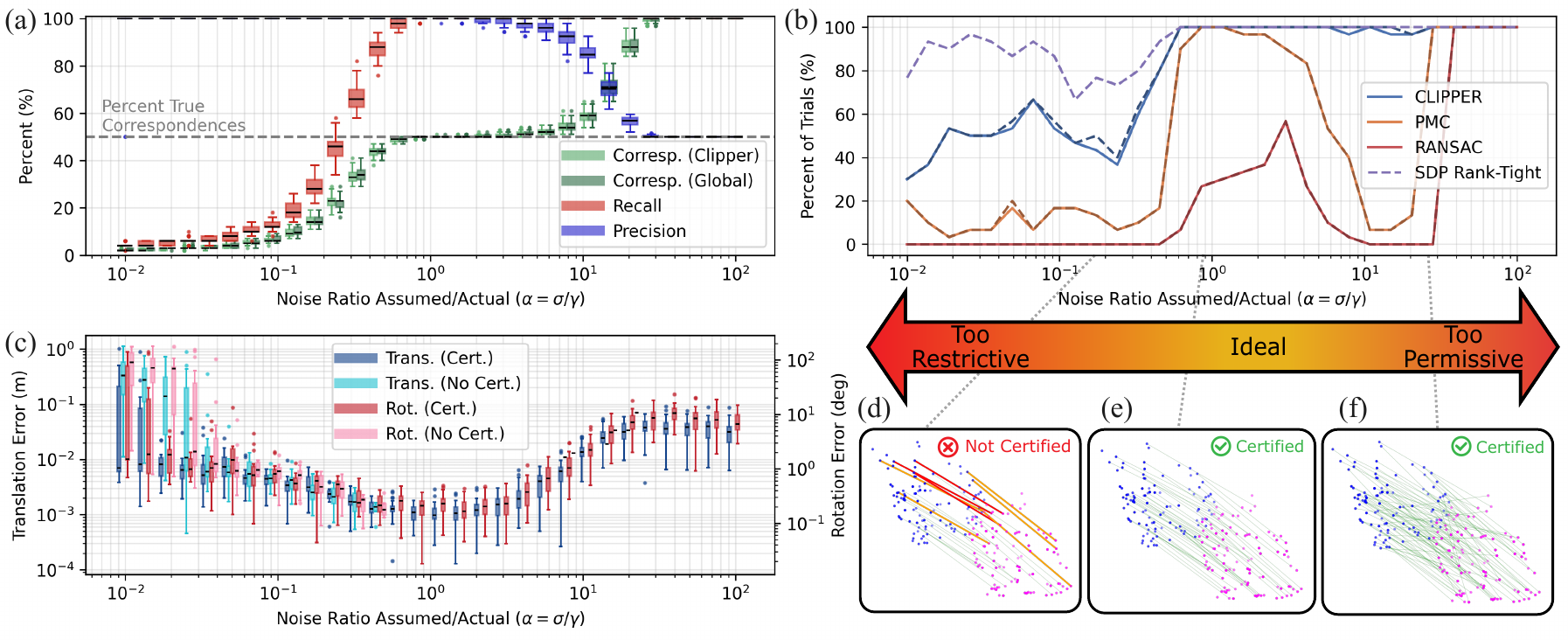}
    \caption{Effect of the data-association score-function parameter $\sigma=\alpha\gamma$ on CLIPPER's global optimality in terms of scaling parameter $\alpha$. (a) shows percent of correspondences selected by CLIPPER compared to the global solution, along with the precision and recall of CLIPPER relative to ground truth. (b) shows percent of (rank-tight) trials certified (solid) versus percent of trials converged to global optimum (dotted) for each method. The percentage of total trials that are rank tight is also provided (purple, dashed). (c) Relative pose-registration error using CLIPPER's inlier correspondences, split by certified and non-certified trials. (d--f) Example trials illustrating an overconfident, restrictive setting that misses inliers and converges to a local minimum ($\alpha<1$), an ideal setting, and an underconfident, permissive setting that selects all associations ($\alpha\geq30$); false-positive and false-negative inliers are shown in red and orange, respectively. The ideal range for $\alpha$ is between $0.5$ and $2$ to avoid local minima and provide the best set of inliers for registration.}
    \label{fig:invariant-sweep}
\end{figure*}

\subsubsection{Data Association Score Parameter Experiments}\label{sec:noise-param-exp}

It is tacitly implied in~\cite{luskCLIPPERRobustData2024} that the consistency-graph score-function parameters should be set based on the actual noise present in the pointclouds. However, the actual noise distribution is often unknown in practice. In this section, we turn our attention to the effect of the score-function parameters on the global optimality and accuracy of CLIPPER solutions.

We fix the ratio of $\sigma$ and $\epsilon$ and study the effect of scaling these parameters while keeping the pointcloud noise constant. More specifically, we keep a constant ratio between the parameters, $\epsilon = 5.54 \sigma$, and scale both by a positive value, $\alpha$, such that $\sigma = 0.01\alpha$. Note that the nominal values for this analysis were chosen to match the distribution of noise added to the pointclouds.

We vary $\alpha$ between $10^{-2}$ and $10^{2}$ with logarithmic spacing and perform 30 data-association trials with 100 putative correspondences, half of which are outliers. In each trial, we run and certify CLIPPER with \ac{CPCert} and solve~\eqref{opt:lovasz-theta} using Mosek to verify the results. We then use the CLIPPER associations to perform pointcloud registration using Arun's method~\cite{arunLeastSquaresFittingTwo1987}.

We show aggregate confusion table results across trials for certification in Table~\ref{tab:cert-conf-table}. We restrict the results in the table to trials in which we can recover a global solution from the interior-point solver. This is possible when the \ac{SDP} has a rank-one solution (842 of the 900 trials), allowing us to verify whether a solution was a global or local optimum. We ran \ac{CPCert} on solutions from CLIPPER, Mosek (\ie, the global solution), \ac{PMC}, and \ac{RANSAC}. We adapt the standard \ac{RANSAC} algorithm for pose regression as described in Appendix~\ref{sec:app:ransac-adapt}. The latter two methods were certified using Algorithm~\ref{alg:disc-to-cont}.
\begin{table}[h]
    \centering
    \caption{Confusion Table Results for \ac{CPCert} Certificates of Optimality Across Trials, in Percentage of Rank-Tight Trials. Rank-Tight Trials: 842; Total Trials: 900. (TP: Certified and Global, FP: Certified and Local, TN: Uncertified and Local, FN: Uncertified and Global)}
    \label{tab:cert-conf-table}
    \begin{tabular}{lrrrr}
        \toprule
        Method & TP & FP & TN & FN \\
        \midrule
        CLIPPER & 83.02 & 0.00 & 16.75 & 0.24  \\
        Mosek (Global) & 99.29 & 0.00 & 0.00 & 0.71  \\
        PMC & 50.83 & 0.00 & 49.17 & 0.00  \\
        RANSAC & 22.45 & 0.00 & 77.55 & 0.00  \\
        \bottomrule
    \end{tabular}
\end{table}
We note that~\ac{CPCert} performs as expected across trials; there were no cases where a local solution was certified as global and a negligible percentage of cases where the solution was global but not certified. Since they are not designed to explicitly solve~\eqref{opt:MSRC}, \ac{PMC} and \ac{RANSAC} do not find the global solution as often as CLIPPER.

A more extensive exposition of the results is shown in Figure~\ref{fig:invariant-sweep}. Figure~\ref{fig:invariant-sweep}-(a) shows the relationship between the parameter $\alpha$ and the percent of total correspondences selected by CLIPPER versus the global solution. Recall and precision of the CLIPPER solution with respect to the ground truth are also provided.

Choosing $\alpha>1$ corresponds to an underconfident stance, which makes the data-association graph more permissive and allows more inconsistent correspondences to be permissible in the optimal clique, decreasing precision. Eventually, when $\alpha\geq30$, the globally optimal solution consists of all possible associations (see trial in Figure~\ref{fig:invariant-sweep}-(f)).
Conversely, choosing $\alpha<1$ corresponds to an overconfident stance, which makes the data-association graph more restrictive and causes the optimal clique to miss potential inliers, decreasing recall (see trial in Figure~\ref{fig:invariant-sweep}-(d)).

We also show the relative error of the pose registration for the CLIPPER inliers in Figure~\ref{fig:invariant-sweep}-(c) and separate the results based on whether a certificate was found (`Cert.' versus `No Cert.').
Figures~\ref{fig:invariant-sweep}-(a) and~\ref{fig:invariant-sweep}-(c) suggest that the ideal value of $\sigma$ is between half and two times the actual noise level, $\gamma=0.01$ m. Outside of this region, association and registration error increase considerably.

Finally, Figure~\ref{fig:invariant-sweep}-(b) shows the effect of $\alpha$ on the tightness of the relaxation and the convergence properties of the local solvers. We first note that the percentage of total trials for which the relaxation is rank tight (dashed, purple line) is 100\% whenever $\alpha\geq0.6$ and is above 70\% otherwise.\footnote{Numerically, the \ac{SDP} solution is considered to be rank one if the ratio between its largest and second-largest eigenvalue is greater than $10^6$.} This suggests that the relaxation is tight whenever the score-function parameters are chosen appropriately.

The remaining lines in the plot show the percentage of rank-tight trials for which a given method converges to a global optimum\footnote{We verify global optimality of a trial whenever the relative difference between the local solver cost and the \ac{SDP} cost found by Mosek is less than $10^{-4}$.} (dotted) compared to when each is certified by~\ac{CPCert} (solid). Importantly, we see that there is good agreement between these curves, signifying that~\ac{CPCert} accurately certifies optimality when possible.

Generally, we see that all methods converge to local minima more often when $\alpha$ drops below $0.8$. Above this value, CLIPPER almost always converges to global minima, while \ac{PMC} and \ac{RANSAC} still converge often to local minima.

Examples of different cases are shown in Figures~\ref{fig:invariant-sweep}-(d) to~\ref{fig:invariant-sweep}-(f), in which the CLIPPER solution is compared to the globally optimal relaxation solution. False-positive (incorrect) inliers are shown in red and false-negative (missed) inliers are shown in orange. The presence of local minima is also suggested by the difference in registration error for certified and non-certified solutions in Figure~\ref{fig:invariant-sweep}-(c).

The results of this section suggest that care must be taken when selecting the parameter $\sigma$ for data-association graph generation. In practice, this presents a challenge since the underlying noise distribution of the pointcloud is not usually known \emph{a priori}. For cases where the actual noise is larger than expected, local minima are more common, and certification becomes more important.

\subsubsection{Data Association Runtime Experiment}

In this section, we demonstrate the runtime performance of certifying a local solution with~\ac{CPCert} compared to solving~\eqref{opt:lovasz-theta} directly with Mosek. We study the runtime of \ac{CPCert} and Mosek as the number of associations and ratio of outliers vary, including the runtime of CLIPPER (\ie, a local solver) as a reference point. In terms of the \ac{SDP} size, the number of associations controls the dimension, $n$, and the ratio of outliers controls the number of constraints, $m$, at each dimension. Experiments were performed with $\alpha=1.0$, and 30 trials were performed for each parameter value. For these experiments, all CLIPPER and Mosek solutions were certified as globally optimal by \ac{CPCert}.

The results of this experiment are shown in Figure~\ref{fig:timing-sweep}. It is immediately clear from the plots that \ac{CPCert} certifies solutions two to three orders of magnitude faster than solving the entire \ac{SDP} directly using Mosek.

Figure~\ref{fig:timing-sweep}-(a) shows the relationship between runtime and \ac{SDP} variable dimension (\ie, number of associations). As expected, the relationship for both Mosek and \ac{CPCert} is cubic with respect to the number of associations, since both methods must perform cubic operations on dense $n$-by-$n$ matrices at each iteration (\eg, multiplication, LDL/Cholesky factorization, eigenvalue computations).

Interestingly, Figure~\ref{fig:timing-sweep}-(b) shows that, with 200 associations, \ac{CPCert} runtime changes little with the number of outliers. More importantly, the runtime stays at approximately 100 ms (near real time) even as the number of constraints approaches 20,000. We posit that this constant relationship is due to efficient parallelization and exploitation of sparsity in our indirect linear solver (see Section~\ref{sec:eff-lin-solve}).

\begin{figure}[t]
    \centering
    \includegraphics[width=\columnwidth]{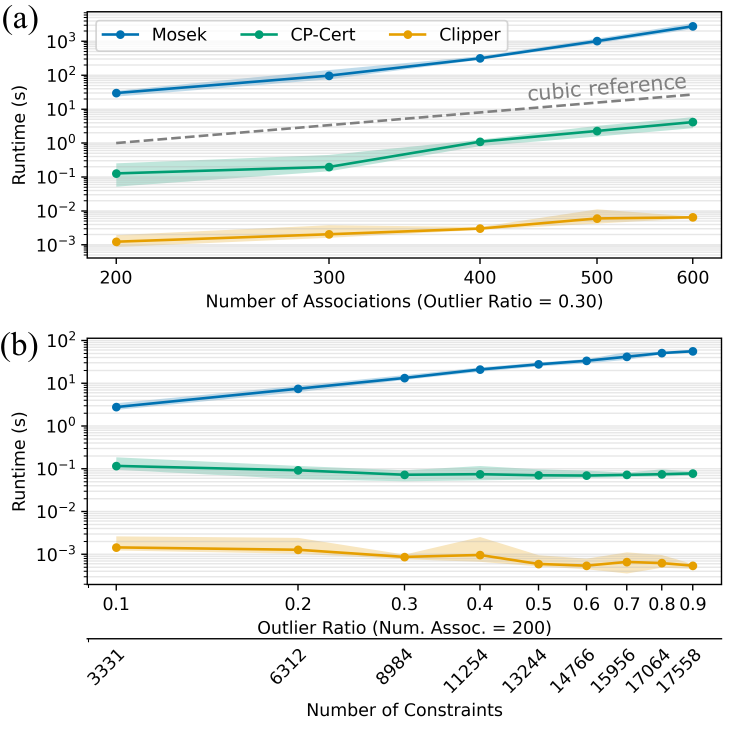}
    \caption{Comparison of Mosek, CLIPPER, and \ac{CPCert} runtimes as different parameters are varied. \ac{CPCert} generally runs two to three orders of magnitude faster than running the full interior-point solve with Mosek. Plotted points correspond to mean runtime and shaded areas encompass all runtimes across trials. (a) compares \ac{CPCert} and Mosek as the number of associations changes with fixed outlier ratio. Both scale cubically with the number of associations (problem dimension). (b) compares \ac{CPCert} and Mosek as the outlier ratio (and corresponding number of constraints) changes for a fixed number of associations (problem dimension). Notably, \ac{CPCert} stays at a near-constant runtime across different numbers of constraints, owing to our efficient linear solver.}
    \label{fig:timing-sweep}
\end{figure}

\subsubsection{Stereo Pose Registration}

In this section, we apply \ac{CPCert} to certify local solutions of the pointcloud-registration problem,~\eqref{opt:mat-weight-reg}, adapted to pointclouds derived from stereo-camera data.
To simulate a stereo-camera setup, we first extract a set of $N$ points from the Stanford Bunny dataset and apply a rigid transformation that maps the points into a `camera' frame.
Following the methodology in Appendix E of~\cite{holmesSemidefiniteRelaxationsMatrixWeighted2024a}, we then derive the (Euclidean) covariance of each point, $\bm{\Sigma}_i$, using a stereo-camera model with a $0.24$ m baseline, a $484.5$-pixel focal length, and an assumed image noise of $0.5$ pixels. This covariance is used to both add noise to the pointcloud (according to~\eqref{eqn:cloud-noise}) and derive matrix weights in the cost function of~\eqref{opt:mat-weight-reg}.

As mentioned above, we use \ac{GTSAM} in Python with custom factors as a local solver for the optimization. We then pass the local solution to \ac{CPCert} for certification. For timing comparison and verification, we also solve the relaxation of~\eqref{opt:mat-weight-reg} using Mosek.

\begin{figure}
    \centering
    \includegraphics[width=\columnwidth]{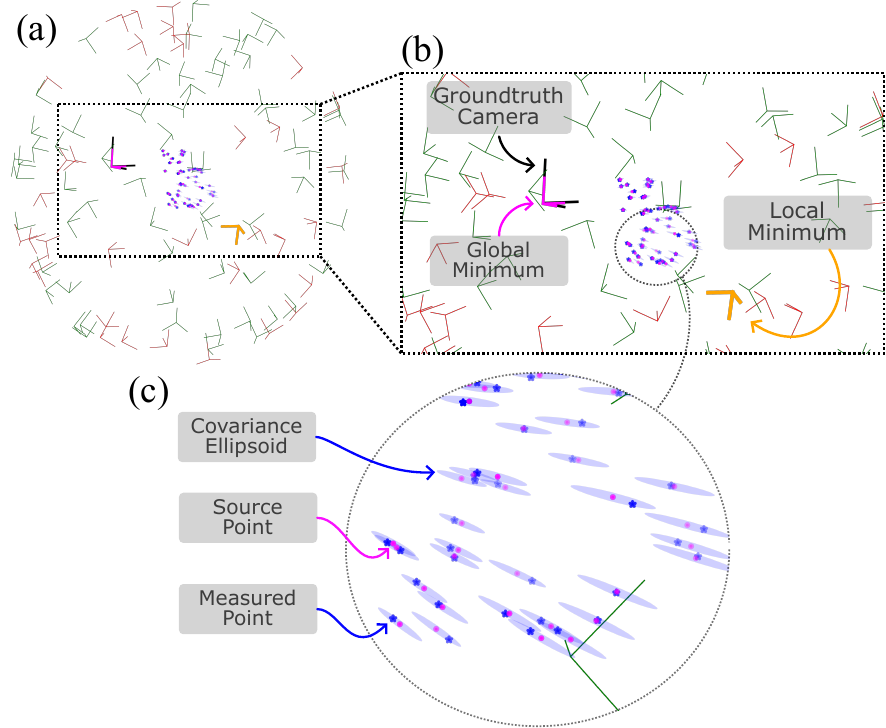}
    \caption{Experimental setup for simulated pose-registration experiments showing initializations and local and global minima. (a) Initial poses -- set in a sphere around the pointcloud -- are colored red if they converge to a local minimum and green if they converge to the global minimum. (b) Ground-truth camera frame shown in black, with the global (magenta) and local (orange) solutions overlaid. (c) Pointcloud details; source points (magenta dots) are corrupted by stereo-camera noise to give the target points (blue stars). Three-standard-deviation uncertainty ellipsoids for measurement noise (blue) elongate in the camera's depth direction.}
    \label{fig:pose-reg-setup}
\end{figure}

Our experimental setup is shown in Figure~\ref{fig:pose-reg-setup}. The camera frame is placed at a distance, $d$, along the camera axis ($z$-axis) away from the centroid of the pointcloud and is rotated by three degrees to provide some parallax effect. In each trial, we randomly initialize the optimization with 100 different poses, such that they are located on a sphere of radius $d$ and their $z$-axis points towards the centroid of the pointcloud.

Table~\ref{tab:pose-reg-conf} shows the confusion table for \ac{CPCert}. We use the \ac{SDP} global minimum cost (found via Mosek) to determine whether a solution returned by \ac{GTSAM} is a global minimum or not. The \ac{SDP} was rank tight in all cases. We note that the certifier successfully identified global minima whenever they appeared.

\begin{table}[h]
    \centering
    \caption{Confusion Table (in percentages) for Pose Registration Certification (TP: Certified and Global, FP: Certified and Local, TN: Uncertified and Local, FN: Uncertified and Global)}
    \label{tab:pose-reg-conf}
    \begin{tabular}{lrrrr}
        \toprule
        Distance (m) & TP & FP & TN & FN\\
        \midrule
        3.0 & 72.0 & 0.0 & 28.0 & 0.0 \\
        5.0 & 87.0 & 0.0 & 13.0 & 0.0 \\
        10.0 & 93.0 & 0.0 & 7.0 & 0.0 \\
        \bottomrule
    \end{tabular}
\end{table}

Table~\ref{tab:pose-reg-runtime} shows runtime results for Mosek and \ac{CPCert}. We separate the runtime results for \ac{CPCert} based on whether a certificate was successfully found or not.

We note that the small size of the pose-registration problem ($n=13,~m=21$) is exactly the regime where we would expect an interior-point solver to perform well. On the other hand, the design considerations of \ac{CPCert} targeted scalability for larger problems. Nevertheless, we see that \ac{CPCert}'s runtimes are approximately five times faster than Mosek's in the success case and have comparable runtimes even when \ac{CPCert} fails to find a certificate (due to local minima).

\begin{table}[h]
    \centering
    \caption{Average Runtime Results for Pose Registration Certification}
    \label{tab:pose-reg-runtime}
    \begin{tabular}{p{15mm}p{15mm}p{15mm}p{15mm}}
        \toprule
        Distance \newline (m)& Mosek\newline (ms) & \ac{CPCert} \newline Succ. (ms) & \ac{CPCert} \newline Fail (ms) \\
        \midrule
        3.0 & 5.2 & 0.9 & 5.6 \\
        5.0 & 4.2 & 1.1 & 5.5 \\
        10.0 & 7.3 & 1.1 & 5.8 \\
    \bottomrule
    \end{tabular}
\end{table}

\subsection{Certifiable Stereo Pipeline}\label{sec:stereo-pipeline}

\begin{figure*}[t]
    \centering
    \includegraphics[width=\textwidth]{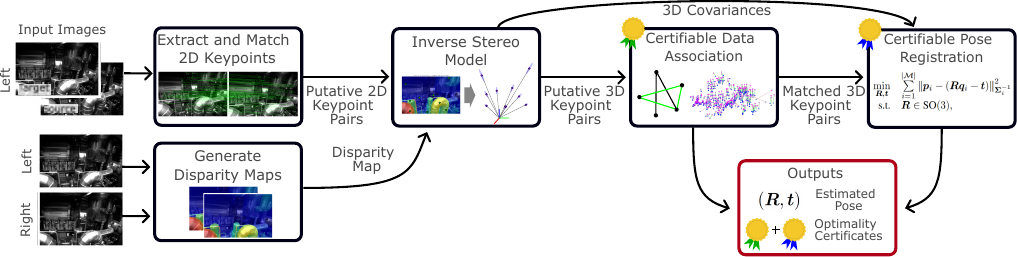}
    \caption{Certifiable stereo pose-registration pipeline. Given left images and disparity maps (from RAFT-Stereo) for two temporally distinct stereo frames, SuperPoint and LightGlue extract and match a putative set of 2D keypoint pairs, which are lifted to 3D via an inverse stereo model. These putative 3D correspondences are used to build a consistency graph, from which CLIPPER and \ac{CPCert} identify a certifiably optimal, outlier-free set of inliers. Finally, \ac{GTSAM} estimates the relative pose between the two frames from this certified inlier set, and \ac{CPCert} certifies the global optimality of the resulting pose estimate. The final result is a certifiably optimal estimated pose based on certified correspondences.}
    \label{fig:pipeline}
\end{figure*}

To validate the performance of \ac{CPCert} on real-world data, we include our certifiable data association and certifiable pose registration in a stereo-camera pose-registration pipeline. A detailed outline of the specific pipeline that we build is shown in Figure~\ref{fig:pipeline}.

As input, the pipeline takes the disparity maps and left images for two temporally distinct frames.
We use RAFT-Stereo to generate the disparity maps for rectified stereo images in a preprocessing step~\cite{lipsonRAFTStereoMultilevelRecurrent2021}, using the same rectification and stereo-camera parameters as~\cite{camposORBSLAM3AccurateOpenSource2021}.
SuperPoint~\cite{detoneSuperPointSelfSupervisedInterest2018} is used to extract 2D features and descriptors from two (left) images from different temporal frames. These features are then matched using LightGlue~\cite{lindenbergerLightGlueLocalFeature2023a} to obtain a putative set of 2D keypoint pairs between the images (limited to 500 pairs). The 2D keypoints are then converted to 3D keypoints using an inverse stereo model and the input disparity maps. These putative 3D keypoint pairs are used to generate a consistency graph (score-function parameters: $\sigma=0.15$ and $\epsilon=0.831$), with which we find a \emph{certifiably optimal set of inliers} using CLIPPER and \ac{CPCert} as discussed in Section~\ref{sec:cert-data-assoc}. This outlier-free, certified set of 3D keypoint pairs is then used to estimate the relative pose between the camera frames of the input images. We solve the pose estimation using \ac{GTSAM} initialized with an identity transformation and certify global optimality using \ac{CPCert}, as described in Section~\ref{sec:cert-reg}.

We apply this pipeline to the Machine Hall 1 (easy) sequence of the EuRoC dataset~\cite{burriEuRoCMicroAerial2016} and present the results in Tables~\ref{tab:pipeline-results} and~\ref{tab:pipeline-runtime} below. We run the pipeline sequentially on each (source) frame in the dataset sequence, selecting the second (target) frame to yield varying time intervals between frames. This allows us to study performance across different magnitudes of relative-pose change between images.

Table~\ref{tab:pipeline-results} shows the certification success rates and absolute registration error for different frame intervals.
From the table, we see that the learned feature extractor is quite good at identifying candidate matches, producing between 16\% and 22\% outliers on average.
Certification almost always succeeds for the data-association problem, and the pose-registration problem is always certified. Pose-registration errors are small, confirming the absence of outliers in the pointcloud data after data association.\footnote{Our focus in this work is the certification of the methods in this pipeline rather than exact tuning of the stereo pipeline for error minimization.}

We investigated certain cases where data-association certification failed by re-solving the \ac{SDP} directly (via Mosek). In these cases, we confirmed that CLIPPER converged to local optima, but these optima were fairly benign and did not lead to catastrophic failure.
Figure~\ref{fig:stereo-pipeline-local-min} demonstrates one such local minimum, as well as the 2D and 3D correspondences and a sample disparity map.

\begin{figure}[t]
    \centering
    \includegraphics[width=\columnwidth]{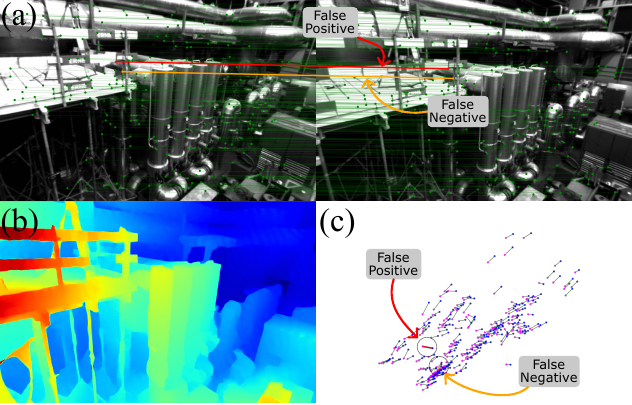}
    \caption{A demonstration of a local minimum returned by CLIPPER that \ac{CPCert} (correctly) does not certify (confirmed by finding the rank-one \ac{SDP} solution). In (a) and (c), green lines indicate correspondences that are correctly accepted by CLIPPER, red lines indicate false positives, and orange lines indicate false negatives with respect to the global solution. (a) shows features and putative correspondences generated by SuperPoint and LightGlue. (b) shows the disparity map generated by RAFT-Stereo and used together with the inverse stereo model to generate the magenta 3D pointcloud in (c). (c) shows data association in the 3D pointcloud, which is used to compute the relative pose.}
    \label{fig:stereo-pipeline-local-min}
\end{figure}

\begin{table}[h]
    \centering
    \caption{Stereo Pipeline Certification Results and Registration Error}
    \label{tab:pipeline-results}
    \begin{tabularx}{\columnwidth}{XXXXXXXX}
    \toprule
    Frame Itvl. (s) & Avg. \# Inliers & Avg. \# Assocs. & Avg. Outlier Rate (\%) & Data Assoc. Cert.  (\%) & Pose Reg. Cert. (\%) & Avg. Trans. Err. (m) & Avg. Rot. Err. (deg) \\
    \midrule
    0.05 & 346.7 & 414.2 & 16.4 & 93.6 & 100.0 & 0.0025 & 0.040 \\
    0.25 & 297.1 & 362.4 & 18.3 & 93.2 & 100.0 & 0.0070 & 0.106 \\
    0.50 & 262.3 & 323.4 & 19.5 & 94.0 & 100.0 & 0.0116 & 0.167 \\
    1.00 & 221.8 & 277.6 & 21.2 & 94.1 & 100.0 & 0.0210 & 0.293 \\
    \bottomrule
    \end{tabularx}
\end{table}

Table~\ref{tab:pipeline-runtime} shows certifier runtimes for both data association and pose registration for different frame intervals. Though pose-registration runtime is consistently around $2$ ms, certificates for data association take up to one second (in successful cases). Runtime improves as the interval increases because the number of putative associations (equal to the \ac{SDP} dimension) drops.

We separate the runtime results based on whether the data-association certificate succeeded or failed. As with the simulated pose-registration problem, we see that the timing results for data association are worse when the certificate fails, due to differences in the early-stopping criterion between success and failure.

\begin{table}[h]
    \centering
    \caption{Stereo Pipeline Certification Runtime Results}
    \label{tab:pipeline-runtime}
    \begin{tabularx}{\columnwidth}{XXXXX}
        \toprule
        Frame Itvl. (s) & Data Assoc. Cert. Succ. (ms) & Data Assoc. Cert. Fail (ms) & Reg. Cert. Succ. (ms) & Reg. Cert. Fail (ms) \\
        \midrule
        0.05 & 981.3 & 4629.3 & 1.2 & -- \\
        0.25 & 615.2 & 4102.9 & 1.3 & -- \\
        0.50 & 466.8 & 2899.9 & 1.4 & -- \\
        1.00 & 351.0 & 2047.3 & 1.5 & -- \\
        \bottomrule
    \end{tabularx}
\end{table}

\section{Conclusion}\label{sec:conclusion}

In this paper, we introduced the Central Path Certifier (\ac{CPCert}), a method for efficiently certifying the global optimality of candidate solutions to the degenerate \ac{SDP} relaxations that are common in robotics and computer vision. Rather than solving the relaxation from scratch or attempting to reconstruct dual variables from the underdetermined \ac{KKT} equations directly, \ac{CPCert} exploits the fact that the dual variables associated with any point on the primal central path are uniquely defined. By perturbing a candidate solution into the interior of the \ac{PSD} cone and applying a specialized, primal-only Newton method to trace the central path back to the candidate, \ac{CPCert} can recover certificates of optimality when the candidate is globally optimal. We showed how to construct a fixed preconditioner that allows the use of a fast conjugate-gradient method that exploits knowledge of the candidate solution, problem sparsity, and parallelism at every step.

We demonstrated two applications of \ac{CPCert} within a certifiable, stereo-based pose-registration pipeline. First, drawing inspiration from the Lov\'{a}sz-theta relaxation used in the optimization literature, we proposed a new \ac{SDP} relaxation of pointcloud data association, and showed empirically that it is both tight on typical robotics problems and well suited to certification via \ac{CPCert}. Second, we applied \ac{CPCert} to certify matrix-weighted pointcloud registration, which properly accounts for the anisotropic depth uncertainty characteristic of stereo measurements. Across both simulated and real-world experiments on the EuRoC dataset, \ac{CPCert} reliably distinguished globally optimal solutions from local minima -- with no observed case of a locally optimal solution being falsely certified -- while certifying solutions two to three orders of magnitude faster than state-of-the-art interior-point solvers. In our full stereo pipeline, this translated to certification of pose registration well within real-time budgets, and certification of data association in near real time in some cases.

A key limitation to deployment on a broader range of problems is that \ac{CPCert} currently relies on a set of parameters that must be tuned to the problem class at hand. Identifying relative, self-scaling parameterizations that generalize across problem types is therefore a natural direction for future work.

Despite the promising runtimes of~\ac{CPCert}, further improvements are required to truly enable real-time certification for data association. A key bottleneck is manipulation and multiplication with the primal variable, $\bm{X}$, which is densely represented, despite being low-rank.
At present, \ac{CPCert} has been designed to handle a single monolithic \ac{PSD} variable, but a promising future direction would involve allowing multiple \ac{PSD} variables that can be handled in parallel.
This would enable the exploitation of so-called chordal decompositions of the primal variables, which have already proven to bring significant improvements to scalability and parallelization of \ac{SDP} solvers~\cite{dumbgenExploitingChordalSparsity2025,zhengChordalFactorwidthDecompositions2021a}. We also plan to explore ways to improve efficiency by leveraging cheaper representations of $\bm{X}$, via matrix sketching~\cite{yurtseverScalableSemidefiniteProgramming2021} or low-rank approximations~\cite{bellaviaRelaxedInteriorPoint2021b}.

\appendices

\section{Proof of Proposition~\ref{prop:rank-tight-suff}}\label{sec:app:proof-rank-tight}

\begin{proof}
    By the assumption of rank tightness, every solution of~\eqref{opt:SDP} has rank at most one, and, since the constraint for which $b_i\not=0$ rules out $\bm{X}=\bm{0}$, exactly one. In particular, the relaxation is tight, since any rank-one solution, $\bm{X}=\bm{x}\bm{x}^T$, provides a point that is feasible for~\eqref{opt:QCQP} with cost $\optValSDP$, so that $\optValQCQP\leq\optValSDP$, which, combined with~\eqref{eqn:weak-duality}, gives equality. Consequently, if $\candx$ is a global solution of~\eqref{opt:QCQP}, then $\candX=\candx\candx^T$ is feasible for~\eqref{opt:SDP} and attains the optimal value, so it lies in the solution set.

    We now show that this solution set must be a singleton. Note that the set is convex, being the intersection of the (convex) feasible set of~\eqref{opt:SDP} with a level set of the linear objective. Suppose that it contains another optimum, $\bm{X}'\not=\candX$, which, by rank tightness, can be written as $\bm{X}'=\bm{y}\bm{y}^T$ for some $\bm{y}\in\mathbb{R}^n$. If $\bm{y}$ is parallel to $\candx$, or equivalently $\bm{X}'=s\candX$ for some $s\geq0$, then feasibility of both matrices with respect to the constraint for which $b_i\not=0$ gives
    \begin{equation*}
        b_i = \inner{\bm{A}_i}{\bm{X}'} = s\inner{\bm{A}_i}{\candX} = s b_i,
    \end{equation*}
    so that $s=1$ and $\bm{X}'=\candX$, contradicting our supposition. Otherwise, $\bm{y}$ and $\candx$ are linearly independent, and any strict convex combination, $(1-t)\candX+t\bm{X}'$ with $t\in(0,1)$, is also a solution, but has rank two, contradicting rank tightness. Therefore, the set of optimal solutions is the singleton, $\{\candX\}$, corresponding exactly to $\candX=\candx\candx^T$.

    Finally, by Lemma 1.1 of~\cite{halickaConvergenceCentralPath2002}, the limit point of the central path as $\mu\rightarrow0$ lies in the solution set of~\eqref{opt:SDP} and, therefore, the central path converges exactly to $\candX$.
\end{proof}

\section{Cost-Constraint Formulation}\label{sec:app:equiv-ipm}
In this section, we prove that the optimizer of~\eqref{opt:CPCert} satisfies all of the properties that we require for our algorithm. We first provide a lemma that characterizes the objective function along the central path.
\begin{lemma}[Objective along the central path]
    \label{lem:bijective-obj}
    The function $\rho(\mu)$ (as defined in Section~\ref{sec:equiv-ipm}) is a strictly increasing bijection,
    \begin{equation}
        \rho(\mu):~(0,\infty)\rightarrow(\optValSDP, \optValAC),
    \end{equation}
    where $\optValAC = \inner{\bm{C}}{\bm{X}_{\subscr{AC}}}$ is the objective value at the analytic center.
\end{lemma}
\begin{proof}
We first show that $\rho(\mu)$ is increasing for all $\mu\in(0,\infty)$.
Fix $0 < \mu_2 < \mu_1$ and let $\bm{X}_i = \bm{X}(\mu_i)$ correspond to distinct points on the central path. Since these points are the unique minimizers of~\eqref{opt:SDP-IPM} for each $\mu_i$~\cite{klerkAspectsSemidefiniteProgramming2004},\footnote{Uniqueness follows from the fact that the log-barrier function is strictly convex~\cite{boydConvexOptimization2004a}.\label{foot:log-conv}} it follows that
\begin{align}
\langle \bm{C},\bm{X}_1\rangle - \mu_1\log\det \bm{X}_1 &< \langle \bm{C},\bm{X}_2\rangle - \mu_1\log\det \bm{X}_2, \label{eq:i}\\
\langle \bm{C},\bm{X}_2\rangle - \mu_2\log\det \bm{X}_2 &< \langle \bm{C},\bm{X}_1\rangle - \mu_2\log\det \bm{X}_1. \label{eq:ii}
\end{align}
Adding \eqref{eq:i}--\eqref{eq:ii} and canceling the cost terms yields $(\mu_2-\mu_1)(\log\det\bm{X}_1 - \log\det\bm{X}_2) < 0$, and since $\mu_2 - \mu_1 < 0$,
\begin{equation}
\log\det\bm{X}_1 - \log\det\bm{X}_2 > 0. \label{eq:detmono}
\end{equation}
Rearranging \eqref{eq:ii} gives
\begin{equation*}
    \langle \bm{C},\bm{X}_1\rangle - \langle \bm{C},\bm{X}_2\rangle > \mu_2(\log\det\bm{X}_1-\log\det\bm{X}_2) > 0,
\end{equation*}
where the last step follows from \eqref{eq:detmono} and $\mu_2 > 0$. It follows that $\rho(\mu_1) > \rho(\mu_2)$.

The fact that $\rho(\mu)$ is a bijection follows from the fact that it is strictly increasing, as long as we restrict its codomain to the interval $(\optValSDP, \optValAC)$ (\ie, the interval of objective values corresponding to $\mu\in(0,\infty)$).
\end{proof}

We will make use of the \ac{KKT} conditions of~\eqref{opt:CPCert} (parameterized on $\epsilon$) in the following form:
\begin{gather}\label{eqn:kkt-eps}
    \bm{C}y_{m+1} + \sum\limits_{i=1}^m\bm{A}_i y_i -\bm{X}(\epsilon)^{-1} = 0,\\
    \inner{\bm{A}_i}{\bm{X}(\epsilon)} = b_i,~\forall~i=1,\dots,m, \\
    \inner{\bm{C}}{\bm{X}(\epsilon)} = \optValSDP + \epsilon\rho_c.
\end{gather}
We will also assume that $\epsilon \in (0, \bar{\epsilon})$ with $\bar{\epsilon}$ as in~\eqref{eqn:eps-bound}. The next lemma establishes a fact that we will require for the main theorem.
\begin{lemma}[Positivity of the cost multiplier]
\label{lem:nu-positive}
For any $\epsilon \in (0, \bar{\epsilon})$, the Lagrange multiplier associated with
the cost constraint in \eqref{eqn:kkt-eps} satisfies $y_{m+1} > 0$.
\end{lemma}

\begin{proof}
Since $\epsilon < \bar{\epsilon}$, the cost constraint of~\eqref{opt:CPCert} and the definition of $\bar{\epsilon}$ give
\begin{equation}
    \label{eq:cost-gap}
    \langle \bm{C}, \bm{X}_{\subscr{AC}} - \bm{X}(\epsilon) \rangle= (\bar{\epsilon} - \epsilon)\rho_c > 0,
\end{equation}
and, in particular, $\bm{X}_{\subscr{AC}} \neq \bm{X}(\epsilon)$. As $\bm{X}_{\subscr{AC}}$ is the \emph{unique} maximizer of $\log\det$ over the feasible set of~\eqref{opt:SDP}\footref{foot:log-conv} and $\bm{X}(\epsilon)$ is feasible, we have
\begin{equation}
\log\det \bm{X}_{\subscr{AC}} > \log\det \bm{X}(\epsilon).
\end{equation}
By strict concavity of the log-determinant, we have
\begin{equation*}
    \log\det \bm{X}(\epsilon) + \inner{\bm{X}(\epsilon)^{-1}}{\bm{X}_{\subscr{AC}} - \bm{X}(\epsilon)} > \log\det \bm{X}_{\subscr{AC}},
\end{equation*}
and canceling $\log\det \bm{X}(\epsilon)$ yields
\begin{equation}
\label{eq:supergrad}
    \inner{\bm{X}(\epsilon)^{-1}}{\bm{X}_{\subscr{AC}} - \bm{X}(\epsilon)}> 0.
\end{equation}
Substituting the stationarity condition in \eqref{eqn:kkt-eps}, we get
\begin{equation*}
y_{m+1} \left\langle \bm{C},\, \bm{X}_{\subscr{AC}} - \bm{X}(\epsilon) \right\rangle > 0,
\end{equation*}
where the $\bm{A}_i$ terms cancel since both matrices are feasible for~\eqref{opt:SDP}.
Combining with \eqref{eq:cost-gap} gives $y_{m+1} > 0$.
\end{proof}

\noindent We now establish the main result:
\begin{theorem}
    For all $\epsilon\in(0,\bar{\epsilon})$, the (unique) solution to~\eqref{opt:CPCert} is on the central path and, as $\epsilon\rightarrow0$, the solution of~\eqref{opt:CPCert} converges to the solution of~\eqref{opt:SDP}.
\end{theorem}
\begin{proof}
To show that solutions of~\eqref{opt:CPCert} are on the central path, it is sufficient to show that the \ac{KKT} conditions of~\eqref{opt:CPCert} and~\eqref{opt:SDP-IPM} are equivalent, since both problems are convex.

Any point on the central path, $\bm{X}(\mu)$, satisfies the \ac{KKT} conditions of~\eqref{opt:SDP-IPM}:
\begin{gather}
    \bm{C} + \sum\limits_{i=1}^m\bm{A}_i \lambda_i -\mu\bm{X}(\mu)^{-1} = 0, \\
    \inner{\bm{A}_i}{\bm{X}(\mu)} = b_i,~\forall~i=1,\dots,m,
\end{gather}
where the $\lambda_i$ are the Lagrange multipliers corresponding to the equality constraints.
Comparing with the \ac{KKT} conditions in~\eqref{eqn:kkt-eps}, we note that the solutions to~\eqref{opt:SDP-IPM} and \eqref{opt:CPCert} can be related as follows:
\begin{gather}\label{eqn:app:mu-epsilon}
    \bm{X}(\mu) = \bm{X}(\epsilon),~\mu = \frac{1}{y_{m+1}},\lambda_i = \frac{y_i}{y_{m+1}},\\
    \rho(\mu) = \optValSDP + \epsilon\rho_c.
\end{gather}
Note that the second expression is valid for all $\mu\in(0,\infty)$ since $y_{m+1}>0$ by Lemma~\ref{lem:nu-positive}.

Under the relation~\eqref{eqn:app:mu-epsilon}, $\epsilon$ is an affine function of $\rho(\mu)$. By Lemma~\ref{lem:bijective-obj}, there exists a bijection between $\epsilon\in(0,\bar{\epsilon})$ and $\mu\in(0,\infty)$, where
\begin{equation}
    \bar{\epsilon}=\frac{\optValAC-\optValSDP}{\rho_c},
\end{equation}
which is exactly the assumed bound of the theorem.
Under the bijection, $\mu\rightarrow0$ is equivalent to $\epsilon\rightarrow0$ and, since $\bm{X}(\mu) = \bm{X}(\epsilon)$, the solution to~\eqref{opt:CPCert} also converges to the solution of~\eqref{opt:SDP}.
\end{proof}

\section{Equivalent Formulations of MSRC}\label{sec:app:equiv-msrc}

In~\cite{luskCLIPPERRobustData2024}, the original version of~\eqref{opt:MSRC} is given as follows:
\begin{equation}\label{opt:MSRC-alt}
\begin{array}{rl}
    \max\limits_{\bm{x} \in \mathbb{R}^n_+} & \bm{x}^T\Affinity\bm{x} \\
    \st & x_i x_j = 0~\mbox{if}~\Affinity_{ij}=0,~\forall~i,j, \\
    & \left\|\bm{x}\right\|^2_2 \leq 1.
\end{array}\tag{MSRC-alt}
\end{equation}
Our formulation of this problem replaces the inequality constraint on the norm of $\bm{x}$ with an equality based on the following theorem:
\begin{theorem}
    Let $\bm{x}^*$ be a globally optimal solution to~\eqref{opt:MSRC-alt}. Then $\left\|\bm{x}^*\right\|^2_2 = 1$.
\end{theorem}
\begin{proof}
    First, note that $\Affinity$ has non-negative elements and positive diagonal by definition. Since the affinity cannot be identically zero, $\Affinity\not=\bm{0}$, the optimum cannot be identically zero, $\bm{x}^*\not=\bm{0}$; we can always select a feasible point, $\bm{x}$, with exactly one non-zero element so that $\bm{x}^{*T} \Affinity \bm{x}^{*}\geq\bm{x}^T \Affinity \bm{x}>0$.

    We prove the theorem by contradiction. Let $\bm{x}^*$ be a globally optimal solution to~\eqref{opt:MSRC-alt} and suppose $0<\left\|\bm{x}^*\right\|^2_2 < 1$.
    Then there exists an $s>1$ such that $\left\|s\bm{x}^*\right\|^2_2 \leq 1$. Moreover, $\forall~i,j$ such that $\Affinity_{ij}=0$, we have
    \begin{equation*}
        (s x^*_i) (s x^*_j) = s^2 (x^*_i x^*_j) = 0.
    \end{equation*}
    Thus $s \bm{x}^*$ is a feasible point for~\eqref{opt:MSRC-alt}. Considering the objective function, we have
    \begin{equation*}
        (s\bm{x}^{*})^T \Affinity (s\bm{x}^{*}) = s^2\bm{x}^{*T} \Affinity \bm{x}^{*} > \bm{x}^{*T} \Affinity \bm{x}^{*},
    \end{equation*}
    since $s^2>1$ and $\bm{x}^{*T} \Affinity \bm{x}^{*}>0$. This contradicts the fact that $\bm{x}^*$ is the global optimum since $(s\bm{x}^{*})$ is a feasible point with a higher objective value.
\end{proof}

\section{Preconditioner Matrix Derivation}\label{sec:app:precond}

Our derivation in this section is based on Section III-B of~\cite{zhangModifiedInteriorpointMethod2017}, but does not require a Cholesky factorization and is specialized to our specific choice of preconditioner.
Substituting $\tilde{\bm{X}}$ into~\eqref{eqn:precond-1} and expanding with the distributive property of the Kronecker product, we get
\begin{align}
    \Precond = \VecCons^T\Big(&\candx\candx^T\otimes\candx\candx^T + \candx\candx^T\otimes \bm{I}\tau \nonumber\\
    &+ \bm{I}\tau\otimes\candx\candx^T + \tau^2 \bm{I}\otimes\bm{I}\Big)\VecCons. \label{eqn:schur-expanded}
\end{align}
We have the following identities:
\begin{equation}
    \begin{gathered}
        \candx\candx^T\otimes\candx\candx^T = (\candx\otimes\candx)(\candx^T\otimes\candx^T), \\
        \VecCons^T(\candx\candx^T\otimes \bm{I})\VecCons = \VecCons^T(\bm{I}\otimes\candx\candx^T)\VecCons,\\
        \candx\candx^T\otimes \bm{I} = (\candx\otimes\bm{I})(\candx\otimes\bm{I})^T,
    \end{gathered}
\end{equation}
where we have applied Lemma 6 from~\cite{zhangModifiedInteriorpointMethod2017} for the second identity. Substituting these expressions into~\eqref{eqn:schur-expanded} gives
\begin{equation*}
    \begin{split}
        \VecCons^T\Big((\candx\otimes\candx)(\candx^T\otimes\candx^T)+2\tau(\candx\otimes\bm{I})(\candx\otimes\bm{I})^T\Big)\VecCons \\
        +\tau^2\VecCons^T\VecCons.
    \end{split}
\end{equation*}
Defining $\bm{V}\in \mathbb{R}^{(m+1)\times (n+1)}$ as
\begin{equation}
    \bm{V} = \begin{bmatrix}
        \VecCons^T(\candx\otimes\candx) & \sqrt{2\tau} \VecCons^T(\candx\otimes\bm{I})
    \end{bmatrix},
\end{equation}
we obtain the desired expression,
\begin{equation}
    \Precond = \tau^2\VecCons^T\VecCons + \bm{V}\bm{V}^T.
\end{equation}

\section{Adapting RANSAC to MSRC}\label{sec:app:ransac-adapt}

Classic RANSAC algorithms for pose registration select the best set of inliers by evaluating the residual error for the registration problem. However, in our context, we seek the set of inliers that maximizes the objective of~\eqref{opt:MSRC}. We implement the RANSAC scoring function to reflect this difference; each potential set of inliers for RANSAC is converted to a continuous solution via Algorithm~\ref{alg:disc-to-cont} and evaluated with respect to the objective. This version of RANSAC, which we call MSRC-RANSAC, is detailed in Algorithm~\ref{alg:msrc-ransac}. We use Arun's method to estimate the registration at each step~\cite{arunLeastSquaresFittingTwo1987} and define the error, $\bm{e}_i=\bm{p}_i - \left(\bm{R}\bm{q}_j + \bm{t}\right)$, as in Section~\ref{sec:cert-reg}.

\begin{algorithm}
    \caption{MSRC-RANSAC}\label{alg:msrc-ransac}
    \begin{algorithmic}[1]
        \Require Putative correspondences $n$, affinity matrix $\Affinity$, minimal sample size $m$, inlier threshold $\epsilon_{\mbox{\scriptsize in}}$, number of iterations $K_{\max}$
        \Ensure Best candidate solution $\hat{\bm{x}}$ to~\eqref{opt:MSRC}, or failure
        \State $\rho^{\mbox{\scriptsize best}} \gets -\infty$ \Comment{initialize best score}
        \State $\hat{\bm{x}} \gets$ failure
        \For{$k = 1, \dots, K_{\max}$}
            \State $\bm{u}_s \gets \mbox{RandomSample}(n, m)$ \Comment{sample min. set}
            \State $\bm{T} \gets \mbox{ArunMethod}(\bm{u}_s)$ \Comment{fit registration}
            \State $\bm{u} \gets \{i : \|\bm{e}_i(\bm{T})\|_2 \le \epsilon_{\mbox{\scriptsize in}}\}$ \Comment{form consensus set}
            \State $\bm{x} \gets \Call{DiscreteToContinuous}{\bm{u}, \Affinity}$ \Comment{via Alg.~\ref{alg:disc-to-cont}}
            \If{$\bm{x} \neq$ failure}
                \State $\rho \gets \objMSRC(\bm{x})$ \Comment{score via~\eqref{opt:MSRC}}
                \If{$\rho > \rho^{\mbox{\scriptsize best}}$}
                    \State $\rho^{\mbox{\scriptsize best}} \gets \rho$ \Comment{update best score}
                    \State $\hat{\bm{x}} \gets \bm{x}$ \Comment{update best candidate}
                \EndIf
            \EndIf
        \EndFor
        \State \Return $\hat{\bm{x}}$
    \end{algorithmic}
\end{algorithm}

\section{Summary of Algorithm Parameters}\label{sec:app:params}

Table~\ref{tab:parameters} summarizes the free parameters of \ac{CPCert}, introduced throughout Section~\ref{sec:cp-cert} and used in Algorithm~\ref{alg:cp-cert}.

\begin{table}[h]
    \centering
    \caption{Parameters of the Central Path Certifier (Algorithm~\ref{alg:cp-cert})}
    \label{tab:parameters}
    \begin{tabular}{c c c c}
        \hline
        \textbf{Symbol} & \textbf{Data Assoc.} & \textbf{Pose Reg.} & \textbf{Default}\\
        \hline
        $\delta$ & $10^{-7}$ & $10^{-5}$  & $10^{-5}$\\
        $\tau$ & $10^{-7}$ & $10^{-5}$ & $10^{-5}$\\
        $K_{\max}$ & 10 & 10 & 10\\
        $\alpha_{\subscr{\mbox{inc}}}$ & 0.1 & 0.1  & 0.1\\
        $\alpha_{\subscr{\mbox{dec}}}$ & 0.9 & 0.9 & 0.9\\
        $\sigma_{\subscr{\mbox{inc}}}$ & 2.0 & 2.0  & 2.0\\
        $\sigma_{\subscr{\mbox{dec}}}$ & 0.6 & 0.6 & 0.6\\
        $\epsilon_{\min}$ & $10^{-10}$ & $10^{-8}$ & $10^{-8}$\\
        $\tau_{\mbox{\scriptsize step}}$ & $10^{-10}$  & $10^{-10}$ & $10^{-10}$ \\
        $\tau_c$ & $ 10^{-5}$ & $10^{-5}$ & $ 10^{-5}$\\
        $\tau_p$ & $ 10^{-5}$ & $10^{-5}$  & $ 10^{-5}$\\
        $\theta_{\max}$ & $ 10^{-2}$ & $8\times10^{-5}$ & $ 10^{-2}$\\
        $\alpha_0$ & 1.0  & 1.0 & 1.0 \\
        $\alpha_{\min}$ & $ 10^{-10}$  & $10^{-10}$ & $ 10^{-10}$ \\
        $\sigma_{\alpha}$ & 0.8  & 0.8  & 0.8 \\
        \hline
    \end{tabular}
\end{table}
\bibliographystyle{IEEEtran}
\bibliography{CentralPathCertifier}
\vfill
\end{document}